\documentclass[numbers]{article}
\usepackage[final]{neurips_2023}
\usepackage[T1]{fontenc}    
\usepackage{hyperref}       
\usepackage{url}            
\usepackage{booktabs}       
\usepackage{amsfonts}       
\usepackage{nicefrac}       
\usepackage{microtype}      
\usepackage{mathtools}
\usepackage[dvipsnames]{xcolor}

\hypersetup{
    colorlinks=true,
    linkcolor=RubineRed,
    citecolor=Emerald,
    filecolor=magenta,      
    urlcolor=Blue,
    pdftitle={Overleaf Example},
    pdfpagemode=FullScreen,
    }
  
\usepackage{amsmath}
\usepackage{amsthm}
\usepackage{wrapfig}
\usepackage{microtype}
\usepackage{graphicx}
\usepackage{adjustbox}
\usepackage{caption}
\usepackage{algorithm}
\usepackage{algorithmic}
\usepackage[vlined,ruled,algo2e]{algorithm2e}
\newcommand{\PyComment}[1]{\ttfamily\textcolor{commentcolor}{\# #1}} 
\newcommand{\PyCode}[1]{\ttfamily\textcolor{black}{#1}}

\newcommand{\E}{\mathbb{E}}

\newcommand{\bL}{\mathcal{L}}

\usepackage{subfigure}
\usepackage{amsmath}
\newtheorem{theorem}{Theorem}[section]
\newtheorem{proposition}[theorem]{Proposition}

\DeclareMathOperator*{\minimize}{minimize}
\DeclareMathOperator*{\maximize}{maximize}

\usepackage[textsize=tiny]{todonotes}
\usepackage{multirow}
\definecolor{commentcolor}{RGB}{110,154,155}
\usepackage{ragged2e}
\usepackage{mathtools, nccmath}

\newcommand\ctn[1]{\textcolor{black}{#1}}
\newcommand\sty[1]{\textcolor{black}{#1}}

\title{Geodesic Multi-Modal Mixup for Robust Fine-Tuning}

\author{%
  Changdae Oh$^*$ \\
  University of Seoul\\
  \And
  Junhyuk So$^*$ \\
  POSTECH \\
  \AND
  Hoyoon Byun \\
  University of Seoul\\
  \And
  YongTaek Lim \\
  University of Seoul\\
  \And
  Minchul Shin \\
  KAIST \\
  \And
  Jong-June Jeon \\
  University of Seoul \\
  \And
  Kyungwoo Song$^+$ \\
  Yonsei University \\
}

\begin{document}

\maketitle

\def\thefootnote{*}\footnotetext{Equal contribution (changdae.oh@uos.ac.kr; junhyukso@postech.ac.kr), $^+$Corresponding author}\def\thefootnote{\arabic{footnote}}

\begin{abstract}
  Pre-trained multi-modal models, such as CLIP, provide transferable embeddings and show promising results in diverse applications. However, the analysis of learned multi-modal embeddings is relatively unexplored, and the embedding transferability can be improved. In this work, \ctn{we observe that CLIP holds separated embedding subspaces for two different modalities, and then we investigate it through the lens of \textit{uniformity-alignment} to measure the quality of learned representation. Both theoretically and empirically, we show that CLIP retains poor uniformity and alignment even after fine-tuning}. Such a lack of alignment and uniformity might restrict the transferability and robustness of embeddings. To this end, we \sty{devise} a new fine-tuning method for robust representation equipping better alignment and uniformity. First, we propose a \textit{Geodesic Multi-Modal Mixup} that mixes the embeddings of image and text to generate hard negative samples on the hypersphere. Then, we fine-tune the model on hard negatives as well as original negatives and positives with contrastive loss. Based on the theoretical analysis \sty{about} hardness guarantee and limiting behavior, we justify the use of our method. Extensive experiments on retrieval, calibration, few- or zero-shot classification (under distribution shift), embedding arithmetic, \ctn{and image captioning} further show that our method provides transferable representations, enabling robust model adaptation on diverse tasks. Code: \url{https://github.com/changdaeoh/multimodal-mixup}
\end{abstract}

\section{Introduction}
\label{s.intro}
\ctn{Witnessing the remarkable success of large-scale pre-training approaches, there have been numerous attempts to build a single \textit{general-purpose model} (so-called foundation model~\cite{bommasani2021opportunities}) rather than casting multiple task-specific models from scratch. Once built, it can be adapted (e.g., fine-tuned) to a wide variety of downstream tasks by leveraging its transferable representation.}
To construct such a general-purpose model,  besides architectural design~\cite{jaegle2021perceiver, jaegle2022perceiver, zhu2022uni} and datasets~\cite{schuhmannlaion, fanminedojo, gadre2023datacomp}, learning methods~\cite{liu2022selfsupervised, haochen2022beyond, wortsman2021robust, kumar2022fine, kirichenko2023last} have been known to be crucial for inducing transferable representation and for adapting the model robustly.

Contrastive Learning (CL)~\cite{van2018representation, he2019moco, chen2020simple} is one of the most popular learning methods that constructs a discriminative embedding space by minimizing distances between positive pairs while maximizing them between negative pairs. \sty{Based on its versatility}, CL has been widely adopted to (pre-) train models on various domains~\cite{chen2020simple, NEURIPS2020_3fe23034, gao2021simcse, somepalli2022saint}. Beyond CL on a single modality, CLIP~\cite{radford2021learning} popularized \textit{multi-modal CL}, which aims to produce close embeddings for paired image-text instances and dissimilar ones for non-paired instances. \sty{Due to its generalization capability and transferability, pre-trained CLIP and its embeddings have been employed on various downstream tasks~\cite{ wang2022clip, ning2023hoiclip, Pratt_2023_ICCV, zhou2023zegclip}.}

However, we observed an unexpected phenomenon: while CLIP's learning objective is designed to align the image and text embeddings explicitly, it has two separate subspaces for each modality, and there are large unexplored interspaces between them as in Fig. \ref{fig:overview_illustration} (illustration) and Fig. \ref{fig:dosnes_embedding} (DOSNES~\cite{lu2019doubly} visualization). We further analyzed this through the lens of \textit{uniformity-alignment}~\cite{wang2020understanding, wang2021understanding}, which is well-studied under uni-modal CL settings but unexplored on multi-modal CLs, and found that CLIP has a poor uniformity-alignment (Fig. \ref{fig:dosnes_embedding} middle) due to its bipartite embedding structure. Theoretically and empirically, we confirmed that this \sty{property} is unchanged even after fine-tuning. As discussed by Wang et al.~\cite{wang2020understanding}, low uniformity-alignment may limit the embedding transferability.
\begin{figure}[t!] 
    \vspace{-1.2em}
    \centerline{\includegraphics[width=\textwidth]{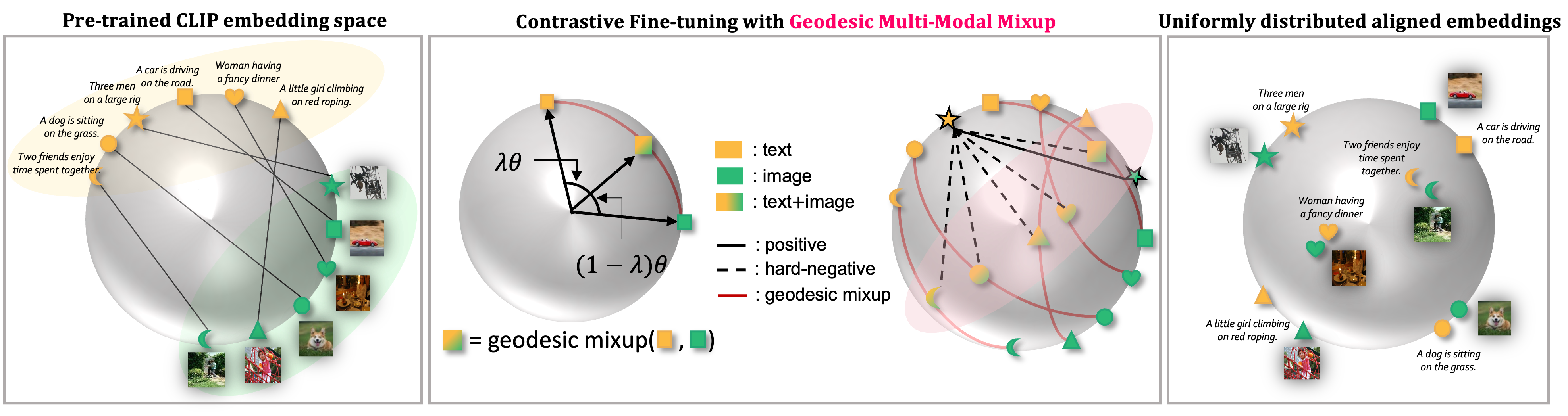}}
    \caption{(Left) \sty{pre-trained CLIP} has two separated clusters for image and text, with a large unexplored interspace. As a result, it has poor uniformity and alignment that might limit the embedding transferability. (Middle) For robust representation, our geodesic multi-modal Mixup ($m^{2}$-Mix) explores the unexploited interspace by mixing the heterogeneous embeddings on hypersphere. Generated samples by $m^{2}$-Mix are regarded as hard negatives for contrastive loss. (Right) As a result, fine-tuning with $m^{2}$-Mix induces robust representation with better transferability.}
    \vspace{-1.7em}
\label{fig:overview_illustration}
\end{figure}
Liang et al.~\cite{liang2022mind}, \ctn{concurrently (in terms of ArXiv preprints)} took a similar observation, \textit{modality gap}, which incurs a bunch of following works, and they also found that increasing temperature ($\tau$) in contrastive loss could somewhat reduce the embedding separation. However, varying the temperature requires manual engineering for each downstream task, and it incurs an inevitable trade-off between uniformity-alignment~\cite{wang2021understanding}. It raises an important question motivating this work: \textit{"How can we obtain a multi-modal representation dealing better with the uniformity-alignment for robust transfer?"}

To answer this, we propose a fundamental learning method, \textbf{geodesic multi-modal Mixup} ($m^{2}$-Mix). As shown in Fig. \ref{fig:overview_illustration}, $m^{2}$-Mix blends the embeddings of different domains, e.g., image and text, and utilizes the mixtures as new entrees of contrastive loss. Our $m^{2}$-Mix contributes to robust representation learning in three perspectives. First, it generates hard negative samples \ctn{which have high similarity with positives (incentivizing alignment)}, and they are known to be crucial for robust CL~\cite{robinsoncontrastive, kalantidis2020hard, wang2021understanding}. We provide a theoretical guarantee of the hardness of $m^{2}$-Mixed samples with empirical results. Second, $m^{2}$-Mix interpolates samples from heterogeneous domains to expand the effective embedding space (increasing uniformity) untapped by pre-trained CLIP, and this is further supported by the limiting behavior of $m^{2}$-Mix (see Prop. \ref{eq:prop}). Third, $m^{2}$-Mix produces virtual augmented samples in embedding space, so it endows a better perception on the out-of-distribution samples as well as in-distribution samples than standard CL.

Contributions: (1) We found that CLIP has a bipartite embedding space with poor uniformity-alignment, which may limit the embedding transferability for downstream tasks. (2) To address this, we propose a new fine-tuning method based on geodesic multi-modal Mixup, $m^{2}$-Mix. To our knowledge, this is the first work that adopts direct mixtures between heterogeneous embeddings. (3) We devised $m^{2}$-Mix as a geodesic Mixup to ensure the mixed embeddings are on a hypersphere, and this is beneficial for stable learning with $L_{2}$-normalized embeddings. (4) We validate our method by theoretical analyses and extensive experiments on retrieval, calibration, full-/few-shot classification (under distribution shift), embedding arithmetic, \ctn{and image captioning}.

\vspace{-0.25em}
\section{Related Works}
\vspace{-0.25em}
\paragraph{Contrastive Representation Learning}
Contrastive Learning (CL) has been widely utilized for representation learning on various domains~\cite{chen2020simple, gao2021simcse, li2019graph, radford2021learning, jia2021scaling}. The goal of CL is to learn an embedding function that maps data into an embedding space so that semantically similar data have close embeddings. Most of CLs adopt $L_{2}$-normalized embedding~\cite{van2018representation, he2019moco, chen2020simple, radford2021learning} for stable learning~\cite{xu2018spherical}, and it makes the embeddings lie on a unit hypersphere. There are key properties to analyze the hyperspherical embeddings, so-called \textit{Uniformity and Alignment}~\cite{wang2020understanding}. Better uniformity and alignment can be regarded as a higher embedding transferability, so it is related to performance on downstream tasks. However, uniformity-alignment analysis on multi-modal settings has not been sufficiently explored yet~\cite{goel2022cyclip}, and we found that CLIP has poor uniformity-alignment before and even after fine-tuning. Meanwhile, it is known that hard negatives for constructing contrastive pairs are necessary to increase the robustness of contrastive representation learning~\cite{NEURIPS2020_f7cade80,zhang2021unleashing,wang2021understanding,zhu2021improving}. To this end, we devise a new approach for multi-modal CL, which generates hard negatives and achieves better uniformity-alignment for \textit{robust and transferable} representation.

\paragraph{Mixup}
\sty{There have been numerous papers that claim} Mixup~\cite{zhang2018mixup} is helpful for robust representation learning and alleviates the overconfident problems and failure under distribution shift as well as the in-distribution accuracy~\cite{thulasidasan2019mixup, han2022umix, han2023reweighted}. Based on such success of Mixup, many works are adopting it as a component of learning frameworks (on vision~\cite{verma2019manifold, yun2019cutmix, kim2021lada, kim2020puzzle}, language~\cite{guo2020nonlinear, yoon2021ssmix, yang2022enhancing} and graph~\cite{wang2021mixup, verma2021graphmix}). Besides, CL with Mixup to help representation learning has also been studied. While traditional CL annotates as 1 for positive pairs and 0 for negative pairs, $i$-Mix~\cite{lee2021imix} and Un-Mix~\cite{shen2022unmix} interpolate the images with ratio $\lambda$, and adopt contrastive loss with pseudo labels according to the mixing ratio $\lambda$. However, there are few works on Mixup for multi-modal learning~\cite{fang2022stemm, cheng2023mixspeech}. STEMM~\cite{fang2022stemm} mixes the speech and text features to align them in shared embedding space, but it cannot be widely used for multi-modal learning because of its architecture-specific design. Therefore, we propose $m^{2}$-Mix, that can be broadly adopted for robust multi-modal representation learning.

\paragraph{Multi-modal Learning} 
\sty{To build a universal intelligence system that simultaneously processes multi-modal data streams}, some focus on developing unified frameworks for multi-modal perception~\cite{jaegle2021perceiver, jaegle2022perceiver, baevski2022data2vec, zhu2022uni, li2023uni}, others aim at representation learning under multi-modal correspondence~\cite{radford2021learning, jia2021scaling, yang2022vision, singh2022flava}. This work focuses on CLIP~\cite{radford2021learning}, the representative multi-modal representation learner. Based on the generalizable embedding space, CLIP has been utilized on numerous tasks~\cite{luo2021clip4clip,fang2021clip2video,zerodetection,textdiffusion}, and there are many attempts to fine-tune it efficiently~\cite{zhou2022learning, zhang2021tip, khattak2022maple, ouali2023black}. Moreover, robust fine-tuning methods have also been actively studied for generalization on both in- and out-of-distribution data. Wortsman et al.~\cite{wortsman2021robust} propose a weight-ensemble method that interpolates the pre-trained and fine-tuned weights, Kumar et al.~\cite{kumar2022fine} adopt a two-stage learning scheme that combines linear probing and fine-tuning, \ctn{and Goyal et al.~\cite{goyal2023finetune} adopt contrastive loss during fine-tuning on image classification.} However, previous works have only focused on uni-modal settings and downstream tasks are restricted to the image classification. This paper provides robust fine-tuning methods for multi-modal settings.

\section{Observations and Problem Define} \label{sec:obs}
For a given batch $D = \{ x_{i},y_{i}\}_{i=1}^{M}$ of $M$ instances where $(x_{i},y_{i})$ denotes the $i$-th image-text pair, the goal of multi-modal learning is to learn the pair-wise relations between image and text. To do this, CLIP~\cite{radford2021learning} learns image encoder $f(\cdot ; \theta_1)$ and text encoder $g(\cdot ; \theta_2)$, so that embedding pairs \sty{$\{I_{i},T_{i}\}_{i=1}^{M}=\{f(x_i;\theta_1), g(y_i;\theta_2)\}_{i=1}^{M}$} get closer to each other. Note that $I_i$ and $T_i$ are $L_{2}$-normalized unit vectors \sty{in most vision-language models}, and they lie on the hypersphere. For simplicity, we omit the learning parameters $\theta_1$ and $\theta_2$ from all the following \sty{equations}. CLIP adopts InfoNCE-style~\cite{van2018representation} loss $C(\cdot, \cdot)$ to enforce the similarity between positive pairs $(x_i,y_i)$ and distance among all remain negative pairs $(x_i,y_j)$. This is formulated as below (Eq. \ref{eq:clip_loss}):
\begin{align}
\tiny
    \vspace{-0.7em}
    C(I,T) = \frac{1}{M}\sum_{i=1}^{M}-\log\frac{\exp({sim}(I_i,T_i)/\tau)}{\sum_{j=1}^{M}\exp{({sim}(I_i, T_j)/\tau)}}  \;\;\;\;\;\; \bL_{\text{CLIP}} = \frac{1}{2}(C(I,T)+C(T,I)) \label{eq:clip_loss}
    \vspace{-0.7em}
\end{align}
Like many other CL approaches, CLIP uses a dot product as a similarity calculation ${sim}(\cdot,\cdot)$ between two vectors and governs $\tau$ as a \textit{learnable} temperature that controls the scale of measured similarity.
Now, we analyze the multi-modal embedding in terms of uniformity and alignment, well-known properties in uni-modal CL literature~\cite{wang2020understanding,wang2021understanding} but unexplored in multi-modal settings. Alignment\footnote{Original formulation of alignment is $ \E_{(x_{i},y_{i})}[\| f(x_{i})-g(y_{i}) \|_{2}^{2}]$, which ignores the similarity among negative pairs. We found that it is not directly related to downstream performances, so we modified the alignment to relative-alignment that handles the similarities of both positive and negative pairs.} (Eq. \ref{eq:alignment}) evaluates the difference between distances (or similarities) of positive pairs compared with the hardest negative pairs, while uniformity (Eq. \ref{eq:uniformity}) indicates how uniformly the data is distributed. The greater alignment and uniformity denote the more transferable representation~\cite{wang2020understanding,wang2021understanding}.

\footnotesize
\begin{align}
\vspace{-0.5em}
\text{Alignment} := & \medmath{-\E_{(x_{i},y_{i})}\big[
\| f(x_{i})-g(y_{i}) \|_{2}^{2}-
\min_{k\neq i}{\| f(x_{i})-g(y_{k}) \|_{2}^{2}}\big]} 
\label{eq:alignment}
\\
\text{Uniformity} := & \medmath{-\log {\E_{(x_{i},y_{j})}
\big[\exp{(-2 \| f(x_{i})-g(y_{j}) \|_{2}^{2})} \big]}} 
\label{eq:uniformity}
\vspace{-0.5em}
\end{align}
\normalsize
Fig. \ref{fig:dosnes_embedding} shows (Left) DOSNES~\cite{lu2019doubly} and (Middle) uniformity-alignment of CLIP's embedding space on image-caption datasets after pre-trained (ZS), naively fine-tuned (FT), and fine-tuned with our method (Ours). Pre-trained CLIP embeddings are separated by two subspaces for image and text with wide interspace between them, and even after fine-tuning, this structure remains unchanged. As a result, CLIP and its fine-tuned embedding have poor uniformity and alignment, and this might limit the transferability and robustness of embeddings to downstream data. \sty{From this observation, we aim to learn a more transferable and robust multi-modal representation with \textit{multi-modal mixup}.} To validate the transferability and robustness of embeddings, besides evaluating uniformity-alignment, we perform diverse downstream tasks: retrieval, few-/zero-shot classification on in-distribution (ID), out-of-distribution (OOD), classification under modality missing, embedding arithmetic, and \ctn{image captioning} in Section \ref{sec:result}. Here, our method produces well-aligned and uniformly distributed embeddings resulting in consistent downstream performance gains (Fig. \ref{fig:dosnes_embedding} Right\footnote{For each method, FT and Ours, we compute scores for each bar by averaging all values of FT and $m^{3}$-Mix in corresponding Tables, except OOD Acc computed by averaging values of WiSE-FT, LP-FT, and \ctn{MaPLe.}}).
\begin{figure}
\vspace{-1.5em}
 \centering
 \includegraphics[width=0.98\linewidth]{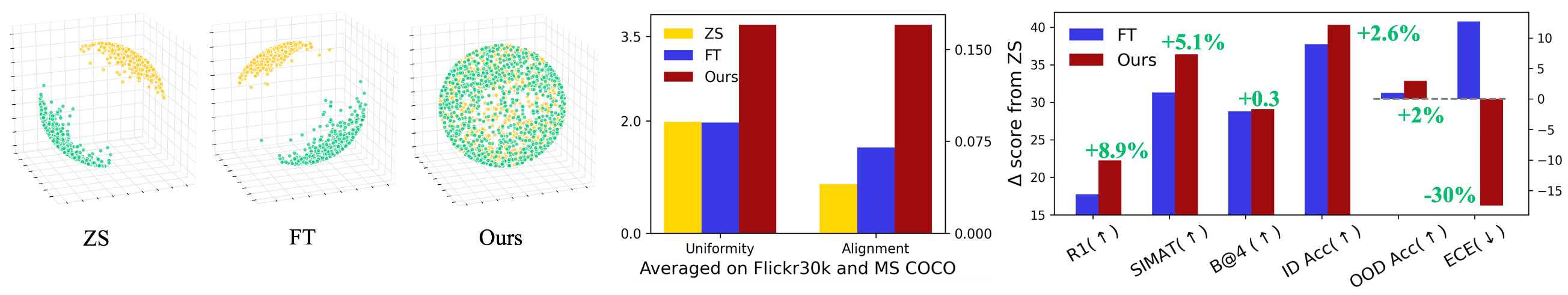} 
\caption{(Left) DOSNES~\cite{lu2019doubly} visualization of CLIP's embedding space on Flickr 30k~\cite{plummer2015flickr30k}. Greens and oranges denote the image and text embedding, respectively. Embeddings of pre-trained (ZS) and naively fine-tuned (FT) ones have two separate clusters with low uniformity and alignment (Middle), which may limit embedding transferability. Multi-modal mixup induces more aligned and uniformly distributed embeddings, (Right) which largely improves downstream performance, including retrieval (R1), embedding arithmetic (SIMAT), \ctn{image captioning (BLUE@4)}, and classification (ID Acc. and OOD Acc.), and uncertainty calibration (ECE).}
\label{fig:dosnes_embedding}
\vspace{-1.5em}
\end{figure}
\section{Methodology}
\subsection{Understanding Geodesic Multi-Modal Mixup} 
From our finding that CLIP embedding is not robust enough, we improve it by fine-tuning CLIP via contrastive loss with virtual hard negative samples. First, we generate hard negatives by mixing the image and text embeddings via \textbf{multi-modal Mixup}, $m^{2}$-Mix. Note that CLIP's $L_{2}$-normalized embeddings lie on a hypersphere, and the mixed ones are \sty{also desirable} to lie on that hypersphere. However, original Mixup~\cite{zhang2018mixup, verma2019manifold} does not guarantee that mixed data lie on the hypersphere. Therefore, we devise a new type of Mixup, \textbf{geodesic Mixup} (defined as Eq. \ref{eq:gmmix}). Geodesic Mixup interpolates two data points on geodesic path, so it ensures that mixed samples lie on the hypersphere.
\footnotesize
\begin{flalign}
    \sty{m_{\lambda}(\vec{a},\vec{b})} &= \vec{a}\frac{\text{sin}(\lambda \theta)}{\text{sin}(\theta)} + \vec{b}\frac{\text{sin}((1-\lambda) \theta)}{\text{sin}(\theta)}, \;\;\; \text{where} \; \theta = \text{cos}^{-1}(\vec{a} \cdot \vec{b}) 
    \text{ and } \lambda \sim \text{Beta}(\alpha,\alpha)
    \label{eq:gmmix}
\end{flalign}
\normalsize
Where $\lambda$ is handled by a hyperparameter $\alpha$ in Beta distribution. It is well-known that $L_{2}$-normalized embeddings are crucial for metric learning~\cite{schroff2015facenet, liu2017sphereface} thus adopted by most of the modern contrastive learning methods~\cite{he2019moco, chen2020simple, radford2021learning}. Therefore, it is necessary that mixture samples lie on the unit sphere, which is guaranteed by our geodesic Mixup. Comparison between geodesic Mixup and standard Mixup following manual $L_{2}$-normalization in Tab. \ref{tab:abla_linear_tau}. Then, we \sty{utilize} the generated hard negatives for contrastive loss by replacing the original negatives (See Eq. \ref{eq:m2_mix_loss}). Here, we only change the denominator term and retain the numerator that compares the similarity between the positive pairs. 
\footnotesize
\begin{align}
C_{m^{2}}(I,T) = \frac{1}{M}\sum_{i=1}^{M}-\log\frac{\exp(I_i\cdot T_i /\tau)}{\exp(I_i\cdot T_i /\tau) + \sum_{j \neq i}\exp(I_i\cdot m_{\lambda}(I_j,T_j)/\tau)} \label{eq:m2_mix_loss} \\
\bL_{m^{2}\text{-Mix}} = \frac{1}{2}(C_{m^{2}}(I,T)+C_{m^{2}}(T,I)) \nonumber
\end{align}
\normalsize
Because CLIP has two-sided polarized embeddings, the similarity between the original image (or text) embedding and mixed embedding is larger than that between the original image and text embeddings. Therefore, $m^{2}$-Mix generates harder negative samples compared with original negatives. The existence of hard negative samples is known to be important in uni-modal CLs~\cite{wang2020understanding,wang2021understanding}, and we validate it under multi-modal CL settings in Section \ref{sec:result}. Besides, we provide a theoretical result on hardness guarantee and limiting behavior of $\bL_{m^{2}\text{-Mix}}$ in the following paragraph.

\paragraph{Theoretical Analysis} In Section \ref{sec:obs}, we observed that naive fine-tuning with standard contrastive loss could not reduce the embedding separation. We speculate this limitation is derived from (1) lack of hard negative samples and (2) vanished learnable $\tau$ (0.01) in $\bL_{\text{CLIP}}$. As shown by Wang et al.~\cite{wang2021understanding}, when $\tau$ approaches zero, the contrastive loss behaves like \textit{triplet loss} with zero-margin. This means that if the similarity between positive samples is greater than that of the nearest negative, there are no incentives to further pull or push the embeddings of positive and negative samples. Moreover, \sty{due to CLIP's bipartite embedding space, it might lack sufficient hard negative samples} to incentivize the models to pursue more alignment and uniformity. Therefore, we argue that \textbf{hard negatives are necessary for multi-modal CL} when there is an embedding space modality gap~\cite{liang2022mind}.

\begin{theorem}[Hardness of $m^2$-Mixed samples]
Let's assume that two random variables $x_{1}$ and $x_{2}$ follow the $M_{d}(\mu_{1},\kappa)$ and $M_{d}(\mu_{2},\kappa)$, von Mises–Fisher distribution with mean direction $\mu_{1}, \mu_{2}$ and concentration parameter $\kappa$ in $\mathbb{R}^{d}$, respectively.
Let $\widetilde{x}=x_{1}+x_{2}$ and $d=2$.
Then, $D_{KL}(p(x_{1})||p(\widetilde{x})) \leq D_{KL}(p(x_{1})||p(x_{2}))$ for sufficiently large $\kappa$.
\label{eq:thm}
\end{theorem}
Theorem \ref{eq:thm} shows that KL-divergence between the pair of an original sample and a mixed sample is less (more confused with positive) than that of another original sample (proof in Supp. D). Meanwhile, because the converged $\tau$ of CLIP is significantly small (i.e., 0.01), it will be reasonable to consider an extreme case: when $\tau \to 0^{+}$. Based on Proposition \ref{eq:prop}, we argue that ones can explicitly enforce uniformity-alignment in multi-modal contrastive learning by equipping $m^2$-Mix with $\bL_{\text{CLIP}}$.

\begin{proposition}[Limiting behavior of $\bL_{\text{CLIP}}$ with $\bL_{m^{2}\text{-Mix}}$]
For sufficiently large $M$, as the temperature of contrastive loss $\tau \to 0^{+}$, the $\bL_{\text{CLIP}}$ and $\bL_{m^{2}\text{-Mix}}$ converges to the triplet loss with zero-margin (i.e., corresponding to negative \text{Alignment}) and negative \text{Uniformity}, respectively. That is: 
$\lim_{\tau \to 0^{+}} \bL_{\text{CLIP}} + \bL_{m^{2}\text{-Mix}} \simeq -(\text{Alignment}+\text{Uniformity})$ 
\label{eq:prop} 
\end{proposition}
$m^{2}$-Mix \sty{brings} two advantages on multi-modal CL with theoretical grounds. Firstly, by generating sufficiently hard negatives, it incentivizes the model to enforce alignment more strongly whether there exists a modality gap or not. Besides, the uniformity is also explicitly increased as $\bL_{m^{2}\text{-Mix}}$ asymptotically converges to negative uniformity. Thus, \sty{$\bL_{\text{CLIP}}$} with $m^{2}$-Mix induces well-aligned and uniformly distributed embeddings, so it makes the model robustly works on diverse tasks. 
\begin{figure}
    \vspace{-1.5em}
    \centering
    \includegraphics[width=\linewidth]{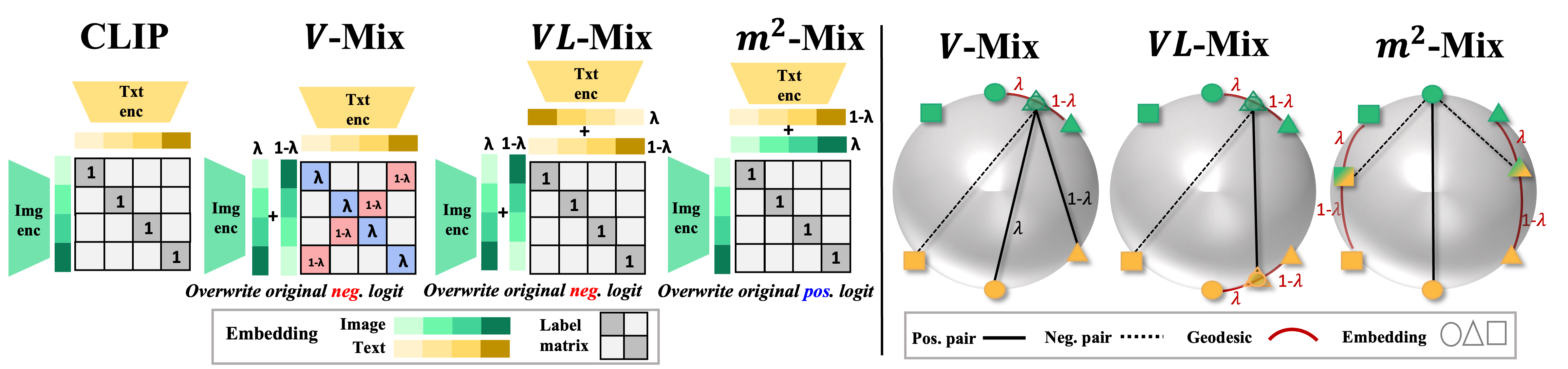}
    \vspace{-1.0em}
    \caption{Comparison among contrastive loss\sty{es}. CLIP enforces the pair-wise similarity between matched image-text embedding pairs. $m^{2}$-Mix generates hard negative samples via mixing two heterogeneous embeddings. We additionally propose uni-modal Mixups, $V$-Mix (and $L$-Mix) and $VL$-Mix, that augment the homogenous embeddings in multi-modal contrastive loss.}
    \label{fig:all_mix_structure}
    \vspace{-1.5em}
\end{figure}
\subsection{Uni-Modal Mixup for CLIP}
$m^{2}$-Mix generates hard negatives for CL to align and distribute embeddings uniformly. Meanwhile, Mixup is known to alleviate the overconfidence problem in uni-modal setups~\cite{thulasidasan2019mixup}. So, we further propose three uni-modal Mixups, Vision-Mix ($V$-Mix), Language-Mix ($L$-Mix), and Vision-Language Mix ($VL$-Mix) to enhance multi-modal CL. Fig. \ref{fig:all_mix_structure} shows the overall structures. 
\paragraph{$uni$-Mix} $V$-Mix and $L$-Mix interpolate the image and text embedding, respectively. To be specific, $V$-Mix mixes the embeddings of images in batch and a flipped (reversed) batch with a ratio $\lambda$. Then, $m_{\lambda}(I_{i},I_{i'})$ has information from $i$ and $i'=M-i$ indexed samples with $\lambda$ and $1-\lambda$ fraction, respectively. Thus, pseudo label for $(m_{\lambda}(I_{i},I_{i'}), T_i)$ pair is $\lambda$ while, that for $(m_{\lambda}(I_{i},I_{i'}), T_{i'})$ is 1-$\lambda$. 

\footnotesize
\begin{flalign}
    C_{V}(I,T) &= \frac{1}{M}\sum_{i=1}^{M}-\lambda\log\frac{\exp( m_{\lambda}(I_{i},I_{i'})\cdot T_i /\tau)}{\sum_{j=1}^{M}\exp(I_i\cdot T_j/\tau)} \nonumber -(1-\lambda)\log\frac{\exp(m_{\lambda}(I_{i},I_{i'}) \cdot T_{i'}/\tau)}{\sum_{j=1}^{M}\exp(I_i\cdot T_j/\tau)}
    \nonumber
\end{flalign} \label{eq:v_mix}
\normalsize
$L$-Mix has the same formula with $V$-Mix, except that it is applied to text-side. The combination loss term of $V$- and $L$-Mix is defined as \ctn{$\bL_{uni\text{-Mix}} = \frac{1}{2}(C_{V}(I,T)+C_{V}(T,I))+\frac{1}{2}(C_{L}(I,T)+C_{L}(T,I))$}
\paragraph{$VL$-Mix} For pair-wise similarity contrast, $V$-Mix and $L$-Mix only mix the image and text embedding, respectively. We additionally propose $VL$-Mix that mixes the image and text embedding simultaneously. Note that $m^{2}$-Mix mixes embeddings of image and text, while $VL$-Mix independently mixes them. Both $m_{\lambda}(I_{i},I_{i'})$ and $m_{\lambda}(T_{i},T_{i'})$ has $i$-th component and $i'$-th component with fraction $\lambda$ and $1-\lambda$ respectively, so the pseudo label for $(m_{\lambda}(I_{i},I_{i'})$, $m_{\lambda}(T_{i},T_{i'}))$ is 1. \ctn{Here, similarities between negative pairs are retain with that of original negatives likewise $uni$-Mix.}

\footnotesize
\begin{flalign}
C_{VL}(I,T) = \frac{1}{M}\sum_{i=1}^{M}-\log\frac{\exp(m_{\lambda}(I_{i},I_{i'})\cdot m_{\lambda}(T_{i},T_{i'})/\tau)}{\sum_{j=1}^{M}\exp(I_i\cdot T_j/\tau)} \nonumber \;\; \bL_{VL\text{-Mix}} = \frac{1}{2}(C_{VL} (I,T)+C_{VL}(T,I)) \nonumber
\end{flalign}
\normalsize
\paragraph{$m^{3}$-Mix} We name the combination of $m^{2}$-Mix with uni-Mix and $VL$-mix as $m^{3}$-Mix, multiple multi-modal Mixup. \sty{Complete objective function is denoted as (weights for each term are omitted):}
\begin{eqnarray}
    \bL_{m^{3}\text{-Mix}}=\bL_{CLIP} + \bL_{m^{2}\text{-Mix}}\nonumber + \bL_{uni\text{-Mix}} + \bL_{VL\text{-Mix}} \nonumber
\end{eqnarray}
\section{Results}
\label{sec:result}
\paragraph{Settings} Unless otherwise stated, we adopt CLIP ViT-B/32 as our backbone model. We consider the following methods as our baselines in retrieval and embedding arithmetic tasks: zero-shot inference (ZS) of CLIP, embedding shift (ES)~\cite{liang2022mind}, naive fine-tuning (FT), and its increased $\tau$ variants, recent Mixup-augmented uni-modal CL methods $i$-Mix~\cite{lee2021imix} and Un-Mix~\cite{shen2022unmix}. \sty{We put task-specific setups on each section and Sec. \ref{appendix:exp_setting} of Supplementary Material (SM). Further details, hyperparameters selection, pseudo code, and additional results are put in Sec. \ref{appendix:exp_setting}, \ref{appendix:pseudo}, and \ref{appendix:results} of SM, respectively.}

\subsection{Generated Samples by $m^{2}$-Mix}
\label{sec:results_sample}
\begin{figure}[!bhtp]
    \vspace{-0.3em}
     \centering
     \subfigure[Retrieved images]{\includegraphics[width=0.48\linewidth]{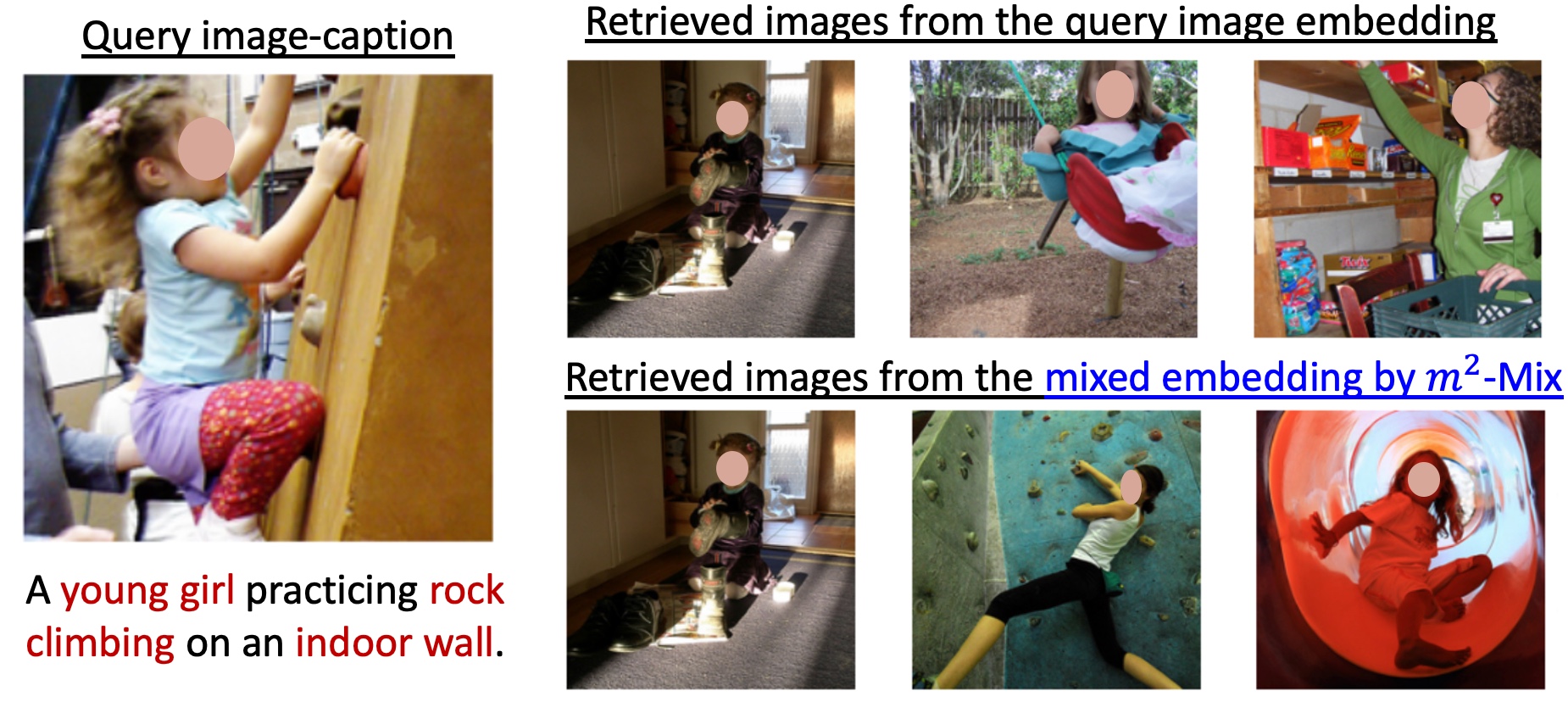}} \label{fig:m2_mix_retrieval}
     \subfigure[Initial epoch]{\includegraphics[width=0.225\linewidth]{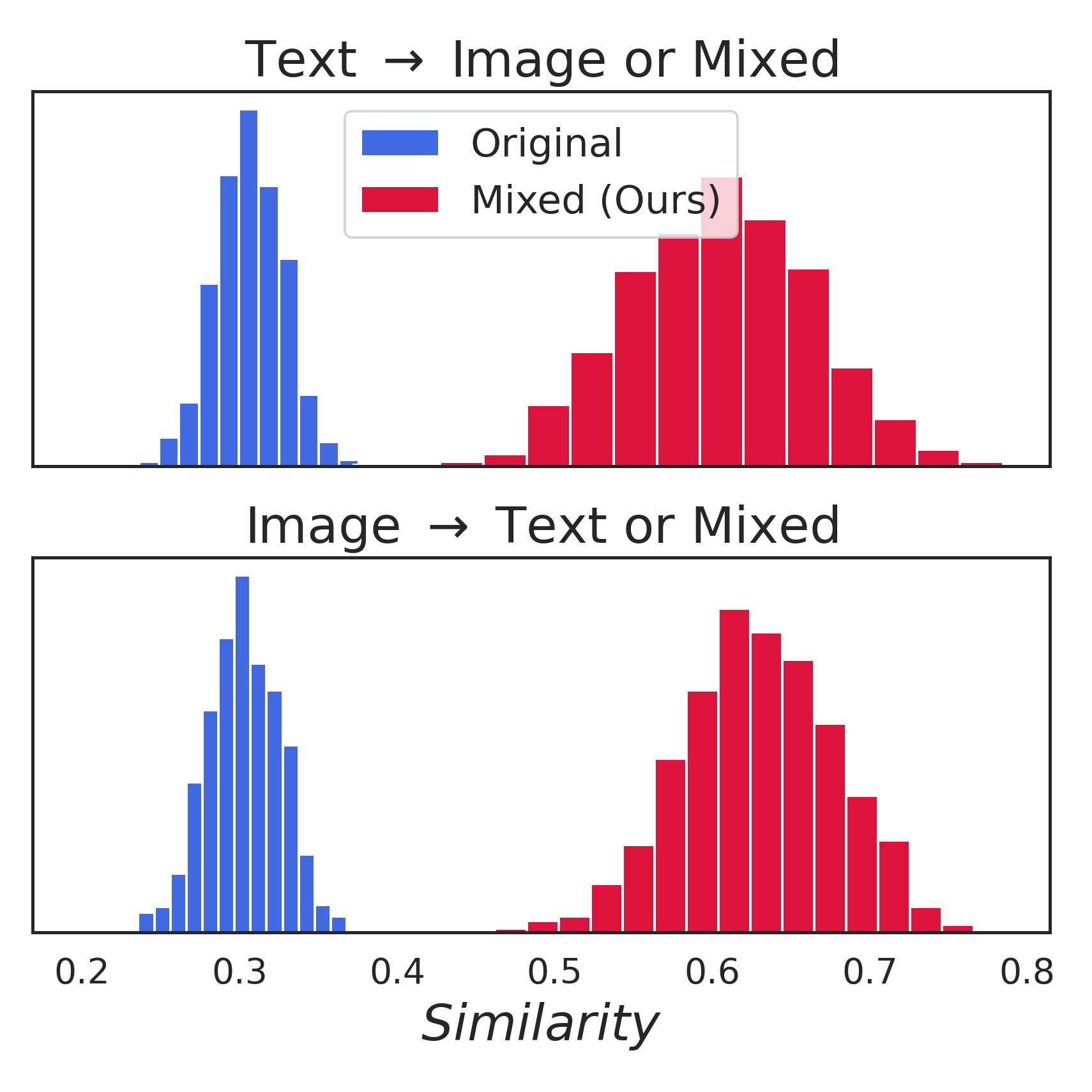}} \label{fig:hardness_hist_zs}
     \subfigure[Last epoch]{\includegraphics[width=0.225\linewidth]{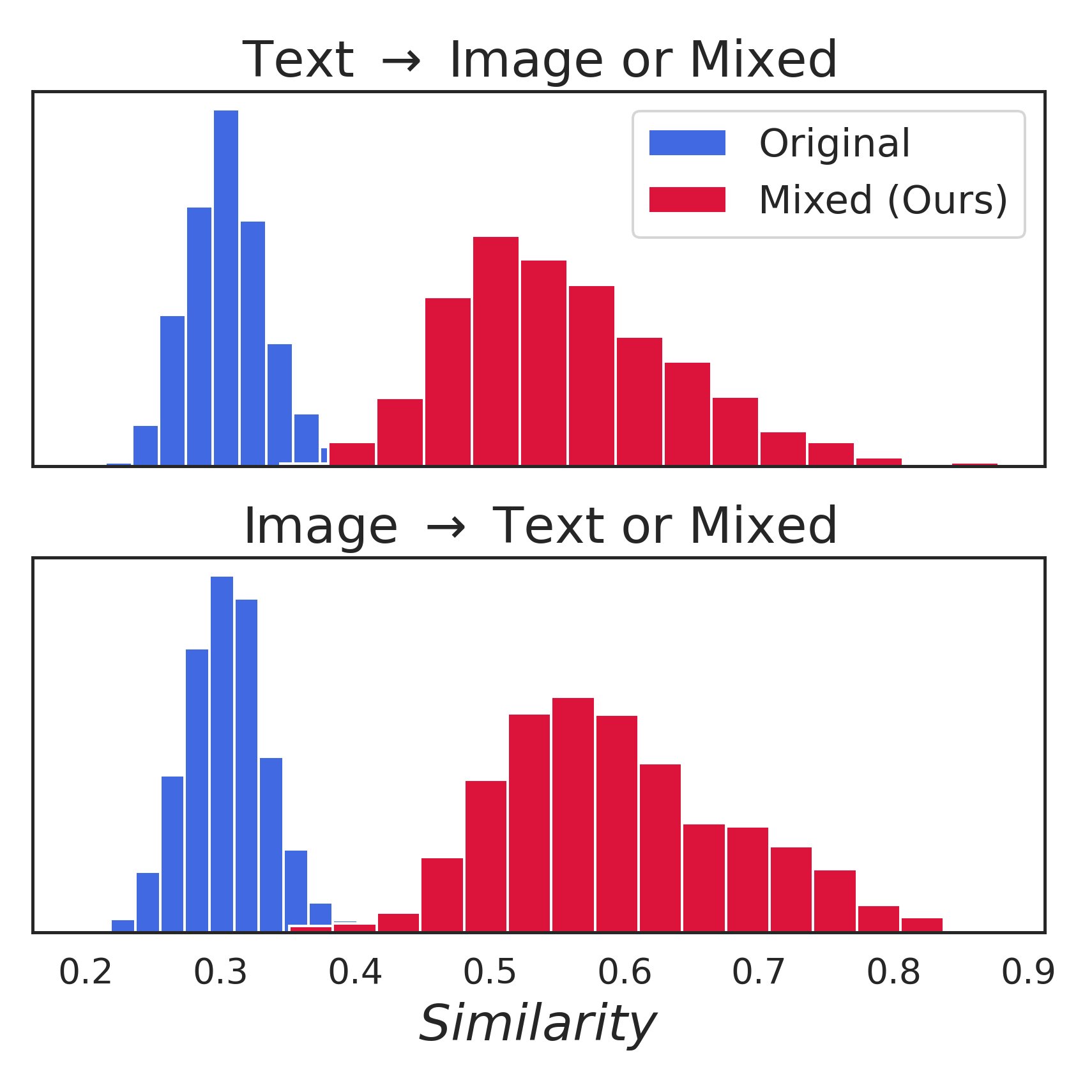}} \label{fig:hardness_hist_ft1}
     \vspace{-0.3em}
    \caption{On Flickr30k, (a) top-3 retrieved images by image embedding (Top) and $m^{2}$-Mixed (Bottom) from a source instance (Left). $m^{2}$-Mix generates an embedding that contains features from both modalities \ctn{(young girl, climbing, and indoor wall), partly lacking} in the image embedding. (b) and (c) denote cosine similarities between given test instances and its top-1 nearest negative embedding during training epochs. $m^{2}$-Mix makes negatives that are highly similar to given instances.
    }
    \vspace{-0.5em}
\label{fig:m2mix_analysis}
\end{figure}
To understand the properties of generated embedding by $m^{2}$-Mix, we explore the mixed embedding from $m^{2}$-Mix. Due to the non-trivial visualization of the embedding itself, we retrieve images that have the most similar embedding with a mixed one. In Fig. \ref{fig:m2mix_analysis} (a), the embedding generated by $m^{2}$-Mix has both features from image and text that lack in original image embedding, e.g., the second and third images from $m^{2}$-Mix have \ctn{rock climbing and indoor wall} represented in the text. Besides, similarity histograms in Fig. \ref{fig:m2mix_analysis} (b) and (c), show that $m^{2}$-Mix consistently produces harder negatives than the original non-mixed counterparts from the initial to last training epochs. 

\subsection{Cross-Modal Retrieval with CLIP}
\label{sec:results_retrieval}
First, we validate our method on image-text retrieval, a representative vision-language task, on Flickr30k~\cite{plummer2015flickr30k} and MS COCO~\cite{lin2014microsoft}. All methods are trained over 9 epochs with Adam optimizer (details in SM). Tab. \ref{tab:retrieval_multimodal} denotes top-1/-5 recall of retrieval. Our $m^{3}$-Mix increases overall performance, while the standard fine-tuning approaches and Mixup-baselines~\cite{lee2021imix, shen2022unmix} have limited performance gain. Corresponding to previous works~\cite{wang2021understanding, liang2022mind}, we also found that properly increasing temperature ($\tau$) in contrastive loss is quite beneficial at improving performance for both FT and $m^{3}$-Mix.

\begin{minipage}[t!]{\textwidth}
\begin{minipage}[]{0.57\textwidth}
\centering
\captionof{table}{Image to text (i$\rightarrow$t) and text to image retrieval (t$\rightarrow$i) retrieval results (top-1/-5 Recall;R1, R5). ZS and FT denote pre-trained and fine-tuned CLIP, respectively.}
\label{tab:retrieval_multimodal}
\tiny
\begin{tabular}{@{}lcccc|cccc@{}}
\toprule
\multirow{3}{*}{} & \multicolumn{4}{c|}{Flickr30k} & \multicolumn{4}{c}{MS COCO} \\
 & \multicolumn{2}{c}{i$\rightarrow$t} & \multicolumn{2}{c|}{t$\rightarrow$i} & \multicolumn{2}{c}{i$\rightarrow$t} & \multicolumn{2}{c}{t$\rightarrow$i} \\
 & R1 & R5 & R1 & R5 & R1 & R5 & R1 & R5 \\ \midrule
ZS & 71.1 & 90.4 & 68.5 & 88.9 & 31.9 & 56.9 & 28.5 & 53.1 \\
ES~\cite{liang2022mind} & 71.8 & 90.0 & 68.5 & 88.9 & 31.9 & 56.9 & 28.7 & 53.0 \\
\midrule
FT & 81.2 & 95.4 & 80.7 & 95.8 & 36.7 & 63.6 & 36.9 & 63.9 \\
FT ($\tau=0.05$) & 82.4 & 95.1 & 82.1 & 95.7 & 40.2 & 68.2 & 41.6 & \textbf{69.9} \\
FT ($\tau=0.10$) & 75.7 & 93.9 & 78.0 & 92.9 & 34.2 & 62.7 & 36.7 & 64.2 \\
\midrule
$i$-Mix~\cite{lee2021imix} & 72.3 & 91.7 & 69.0 & 91.1 & 34.0 & 63.0 & 34.6 & 62.2 \\
Un-Mix~\cite{shen2022unmix} & 78.5 & 95.4 & 74.1 & 91.8 & 38.8 & 66.2 & 33.4 & 61.0 \\
\midrule
$m^{3}$-Mix & 82.3 & \textbf{95.9} & 82.7 & \textbf{96.0} & \textbf{41.0} & \textbf{68.3} & 39.9 & 67.9 \\
$m^{3}$-Mix ($\tau=0.05$) & \textbf{82.7} & 95.7 & \textbf{82.8} & 95.5 & 40.4 & 67.9 & \textbf{42.0} & 68.8 \\
\bottomrule
\end{tabular}
\end{minipage}
\hfill
\begin{minipage}[]{0.40\textwidth}
\centering
\captionof{table}{Retrieval with disjoint uni-modal models. ZS is a naive 
 combination of two models without joint-tuning.}
\tiny
\begin{tabular}{@{}lcccc@{}}
\toprule
\multirow{3}{*}{} & \multicolumn{4}{c}{Flickr30k} \\
& \multicolumn{2}{c}{i $\rightarrow$ t} & \multicolumn{2}{c}{t $\rightarrow$ i} \\
 & R1 & R5 & R1 & R5 \\ \midrule
ZS & 0.1 & 0.4 & 0.1 & 0.2\\
ES~\cite{liang2022mind} & 0.1 & 0.5 & 0.2 & 0.2 \\ \midrule
FT & 28.7 & 61.7 & 26.7 & 59.4  \\
FT ($\tau=0.05$) & 31.5 & 64.2 & 29.2 & 61.4  \\
FT ($\tau=0.10$) & 30.0 & 62.7 & 30.1 & 60.6  \\ \midrule
$i$-Mix~\cite{lee2021imix} & 27.6 &  60.3 & 27.1 & 60.7 \\
Un-Mix~\cite{shen2022unmix} & 31.5 & 64.3 & 29.2 & 61.2 \\ \midrule
$m^{3}$-Mix & 31.9 & 62.6 & 30.3 & 61.0 \\
$m^{3}$-Mix ($\tau=0.05$) & \textbf{32.5} & \textbf{64.7} & \textbf{30.4} & \textbf{63.4} \\ \bottomrule
\end{tabular} \label{table:retrieval_unimodal}
\end{minipage}
\end{minipage}

\begin{wraptable}{h!}{0.40\linewidth}
\centering
\vspace{-1em}
\caption{Calibration on Flickr30k.} 
\label{tab:calibration}
\tiny
\begin{tabular}{lc|cccc}
\toprule
Metric & Task & ZS & FT & $m^{3}$-Mix \\ \midrule
\multirow{2}{*}{ECE ($\downarrow$)} & i $\rightarrow$ t & 1.90 & 2.26 &  \textbf{1.54} \\
 & t $\rightarrow$ i & 1.88 & 2.00 & \textbf{1.58} \\ \bottomrule
\end{tabular} \vspace{-1em}
\end{wraptable}
Besides, we observed the improved uniformity and alignment by $m^{3}$-Mix (Fig. \ref{fig:dosnes_embedding}) not only enhances the Recall of retrievals but also contributes to the calibration~\cite{guo2017calibration}. The left side of Fig. \ref{fig:calibration} denotes the reliability diagrams with calibration errors of the text-to-image retrieval R1 score on Flickr30K. While the naively fine-tuned CLIP has poor calibration, fine-tuning with $m^{3}$-Mix alleviates the overconfidence issue somewhat and results in a better calibration. This is further confirmed by Tab. \ref{tab:calibration}, in which our $m^3$-Mix significantly improves the Expected Calibration Error (ECE) of CLIP. Meanwhile, it is known that the ECE value can be improved by adjusting the temperature $\tau$, i.e., temperature scaling~\cite{guo2017calibration}. Therefore, we provide the sensitive analysis on varying $\tau$. In Fig. \ref{fig:calibration} right side, our method shows relatively robust ECE under varying $\tau$, implying that our multi-modal Mixup-based CL induces the well-calibrated multi-modal model, which is crucial for reliable AI applications.
\begin{figure}[htp!]
\vspace{-0.5em}
  \centering
    \includegraphics[width=0.93\linewidth]{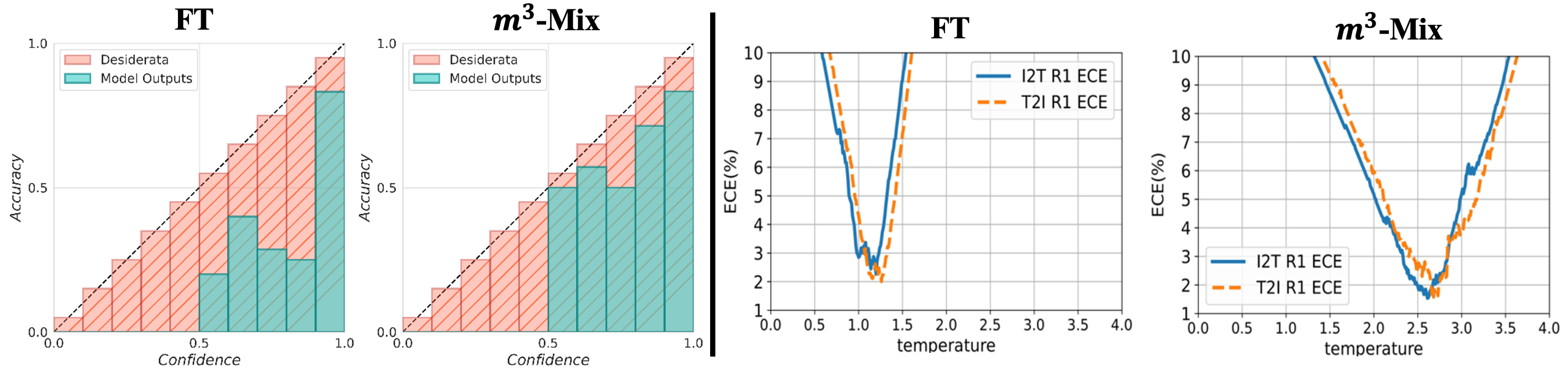}
    \caption{Reliability diagram (left) and ECE under varying temperature $\tau$ (right) on Flickr30K image-to-text retrieval. Our method shows (1) a better reliability diagram (close to $y=x$), (2) achieves a lower minimum ECE value, and (3) more stable across varying $\tau$ than naive fine-tuning. Thus, representation learning by $m^3$-Mix robustly induces a well-calibrated multi-modal model.}
    \label{fig:calibration}
    \vspace{-1em}
\end{figure}
\subsection{Cross-Modal Retrieval with Uni-Modal Pre-Trained Models} \label{sec:ret_uni}
Sometimes, the high-quality annotations for multi-modal datasets are expensive, and there are cases when plenty of paired multi-modal data is unavailable. Then, it is crucial to exploit the uni-modal pre-trained models for learning multi-modal embedding space~\cite{Zhai_2022_CVPR}. To this end, we validate our $m^{3}$-Mix on the fine-tuning of disjointly pre-trained uni-modal models. Specifically, we jointly fine-tune the pre-trained BERT~\cite{kenton2019bert} and ResNet-50~\cite{he2016deep} with a contrastive loss on Flickr30k (in Table \ref{table:retrieval_unimodal}). Among candidates, $m^{3}$-Mix with higher $\tau$ consistently achieves the highest performance so that it can be adopted as an effective joint tuning method for independently pre-trained uni-modal models.

\subsection{Few-Shot Adaptation and Robustness on Distribution Shift}
\begin{minipage}{\textwidth}
\begin{minipage}[]{0.36\textwidth}
\centering
\captionof{table}{Few-shot adaptation under general setting.}
\tiny
\begin{tabular}{lccc|c}
\toprule
\multirow{2}{*}{Method} & \multicolumn{4}{c}{Dataset} \\ & Pets & SVHN & CLEVR & Avg. \\ \midrule
ZS                     & 87.49           & 13.63        & 20.70           & 40.61       \\
FT                     & 89.37           & 45.00        & 53.49           & 62.62       \\
FT w/ $V$-Mix                & 89.45           & 44.61        & 53.93           & 62.66       \\
FT w/ $L$-Mix                & 89.43           & 48.42        & 53.91           & 63.92       \\
FT w/ $VL$-Mix               & 89.56           & 45.22        & 53.75           & 62.84       \\
FT w/ $m^2$-Mix              & 90.05           & 46.24        & 53.60           & 63.29       \\
$m^3$-Mix              & 90.16           & 54.84        & 53.85           & 66.28       \\
$m^3$-Mix ($\tau=0.05$)& 90.49           & 60.90        & 53.95           & 68.45       \\ \midrule
WiSE-FT~\cite{wortsman2021robust}  & 91.80           & 35.04        & 41.93           & 56.25       \\
WiSE-FT w/ $m^3$-Mix   & \textbf{92.51}           & 58.55        & 47.11           & 66.06       \\
LP-FT~\cite{kumar2022fine}   & 89.92           & 44.91        & 53.62           & 62.82       \\
LP-FT w/ $m^3$-Mix     & 91.03    & \textbf{64.24}  & \textbf{55.20}     & \textbf{70.16}  \\
MaPLe~\cite{khattak2022maple}      & 90.87        & 47.62        & 43.05      & 60.51       \\
MaPLe w/ $m^3$-Mix     & 91.14           & 52.72        & 45.20           & 63.02       \\ \bottomrule
\end{tabular}
\label{table:imageclf_transfer}
\end{minipage}
\hfill
\begin{minipage}[]{0.54\textwidth}
    \centering
    \captionof{table}{Few-shot (ImageNet; IN) and zero-shot evaluation under distribution shift (-V2, -A, -R, -S).}
    \tiny
    \begin{tabular}{lc|cccc|c}
    \toprule
        \multirow{2}{*}{Method} & \multicolumn{5}{c}{Dataset} &  \\ 
             & IN    & IN-V2 & IN-A  & IN-R  & IN-S  & Avg.  \\ \midrule
          ZS & 62.06 & 54.80 & 29.63 & 66.02 & 40.82 & 50.67  \\ 
        FT & 65.44 & 55.35 & 20.07 & 58.16 & 34.50 & 46.70  \\
        FT w/ $V$-Mix & 66.00 & 56.19 & 20.85 & 60.50 & 34.97 & 47.70  \\ 
        FT w/ $L$-Mix & 65.96 & 55.95 & 20.57 & 60.54 & 35.25 & 47.65  \\ 
        FT w/ $VL$-Mix & 66.24 & 56.70 & 21.36 & 61.07 & 35.11 & 48.10  \\ 
        FT w/ $m^2$-Mix  & 67.04 & 57.39 & 20.05 & 59.28 & 35.31 & 47.81  \\
        $m^3$-Mix & 67.08 & 57.55 & 20.80 & 60.96 & 35.86 & 48.45  \\
        $m^3$-Mix ($\tau=0.05$) & 68.40 & 58.51 & 22.17 & 62.28 & 37.62 & 49.80  \\ \midrule
        WiSE-FT~\cite{wortsman2021robust} & 69.00 & 59.66 & 28.01 & 64.84 & 41.05 & 52.51  \\ 
        WiSE-FT w/ $m^3$-Mix & \textbf{69.65} & \textbf{60.71} & 29.16 & 66.75 & 42.19 & \textbf{53.69}  \\ 
        LP-FT~\cite{kumar2022fine} & 68.22 & 58.40 & 25.57 & 63.36 & 38.04 & 50.72  \\ 
        LP-FT w/ $m^3$-Mix & 68.62 & 59.17 & 25.85 & 65.14 & 38.78 & 51.51  \\ 
        MaPLe~\cite{khattak2022maple} & 65.59 & 58.44 & 32.49 & 68.13 & 42.53 & 53.44  \\
        MaPLe w/ $m^3$-Mix & 65.76 & 58.16 & \textbf{32.52} & \textbf{68.20} & \textbf{42.67} & 53.46 \\ \bottomrule
    \end{tabular}
    \label{table:imageclf}    
\end{minipage}
\end{minipage}

Next, we evaluate our methods on few-shot image classification under general (in Tab. \ref{table:imageclf_transfer}) and distribution shift settings (in Tab. \ref{table:imageclf} and \ref{tab:abla_linear_tau}). We consider OxfordPets~\cite{6248092}, SVHN~\cite{37648}, and CLEVR~\cite{johnson2017clevr} for the general setting\footnote{Results of CLIP ViT-B/16 on other transfer learning benchmark datasets are provided in Sec. \ref{appendix:results} of SM.} and ImageNet-1k, ImageNetV2~\cite{recht2019imagenet}, ImageNet-A~\cite{hendrycks2021natural}, ImageNet-R~\cite{hendrycks2021many}, and ImageNet-Sketch~\cite{wang2019learning} for distribution shift setting. Unlike MS COCO and Flickr30K, These datasets provide class name labels only and do not have captions corresponding to each image. To make CL methods amendable for this setting, we adopt a common prompt \texttt{'a photo of {classname}'} that wraps the class name with a short context and use this as captions of images. Following~\cite{zhou2022conditional, khattak2022maple}, we perform the tasks under a few-shot evaluation protocol: 16-shot training samples per class and inference on the entire test set.
\begin{wraptable}{!htp}{0.34\textwidth}
\centering
\vspace{-1.3em}
\caption{Ablation study on Mixup.}
\scriptsize
\begin{tabular}{@{}l|cc@{}}
\toprule
\multirow{2}{*}{Temperature ($\tau$)} & \multicolumn{2}{c}{$m^3$-Mix type} \\
 & linear & geodesic \\ \midrule
0.01 & 48.36 & \textbf{48.45} \\
0.05 & 48.48 & \textbf{49.80} \\
0.10 & 45.20 & \textbf{46.41} \\ \bottomrule
\end{tabular}
\label{tab:abla_linear_tau}
\vspace{-1em}
\end{wraptable}
As baselines, we first consider zero-shot CLIP (ZS) and construct the contrastive loss adoption of vanilla fine-tuning (FT). Then, we showcase our methods with exhaustive ablation ($V$-, $L$-, $VL$-, and $m^{2}$-Mix) as well as our complete objective $m^{3}$-Mix with its high-temperature variant. To further compare our approach with state-of-the-art (SOTA) fine-tuning methods, we consider MaPLe~\cite{khattak2022maple} that optimizes the continuous prompts inside the text and image encoders of CLIP, and the contrastive loss extended version of uni-modal fine-tune methods: LP-FT~\cite{kumar2022fine} and WiSE-FT~\cite{wortsman2021robust}.

In both general and distribution shift settings, $m^2$-Mix and uni-modal Mixups ($V$-, $L$-, $VL$-) contribute to boost the few-/zero-shot classification performance. After integrating them, $m^3$-Mix and its high-temperature variant give significant performance improvement, implying $m^3$-Mix is an effective fine-tuning method that covers challenge generalization setup. Moreover, when $m^3$-Mix combined with SOTA fine-tuning methods~\cite{wortsman2021robust, kumar2022fine, khattak2022maple}, it consistently brings performance. Therefore, $m^3$-Mix is a flexible plug-in method that can be collaborated with many algorithms.
Besides, in Tab. \ref{tab:abla_linear_tau} (mean Acc. of ImageNet variants), we show that geodesic Mixup achieves superior results than linear Mixup with manual $L_{2}$-normalization. Thus, based on its analytic property, geodesic Mixup is more suitable than linear Mixup under frameworks that learn representation on a hypersphere, as in modern CLs.

\subsection{Robustness on Modality Missing}
\label{sec:results_gmc}
In this section, we study whether $m^2$-Mix can help the multi-modal representation learning for video recognition (CMU-MOSEI~\cite{bagher-zadeh-etal-2018-multimodal}) under modality missing. Recently, Geometric Multi-modal Contrastive Learning (GMC)~\cite{poklukar2022gmc} achieved competitive results on CMU-MOSEI.

However, GMC only considers the uni-to-joint modal relationship~\cite{poklukar2022gmc}, while our $m^2$-Mix explicitly generates the mixture of bi-modal semantics so that it can additionally consider the bi-to-joint modal relation. 
From this, we hypothesize that $m^2$-Mix can further improve the robustness and informativeness of the representation. For evaluation, we add the $m^2$-Mix on top of GMC.

\begin{wraptable}{t!}{0.33\linewidth}
\vspace{-1em}
  \centering
    \includegraphics[width=0.30\textwidth]{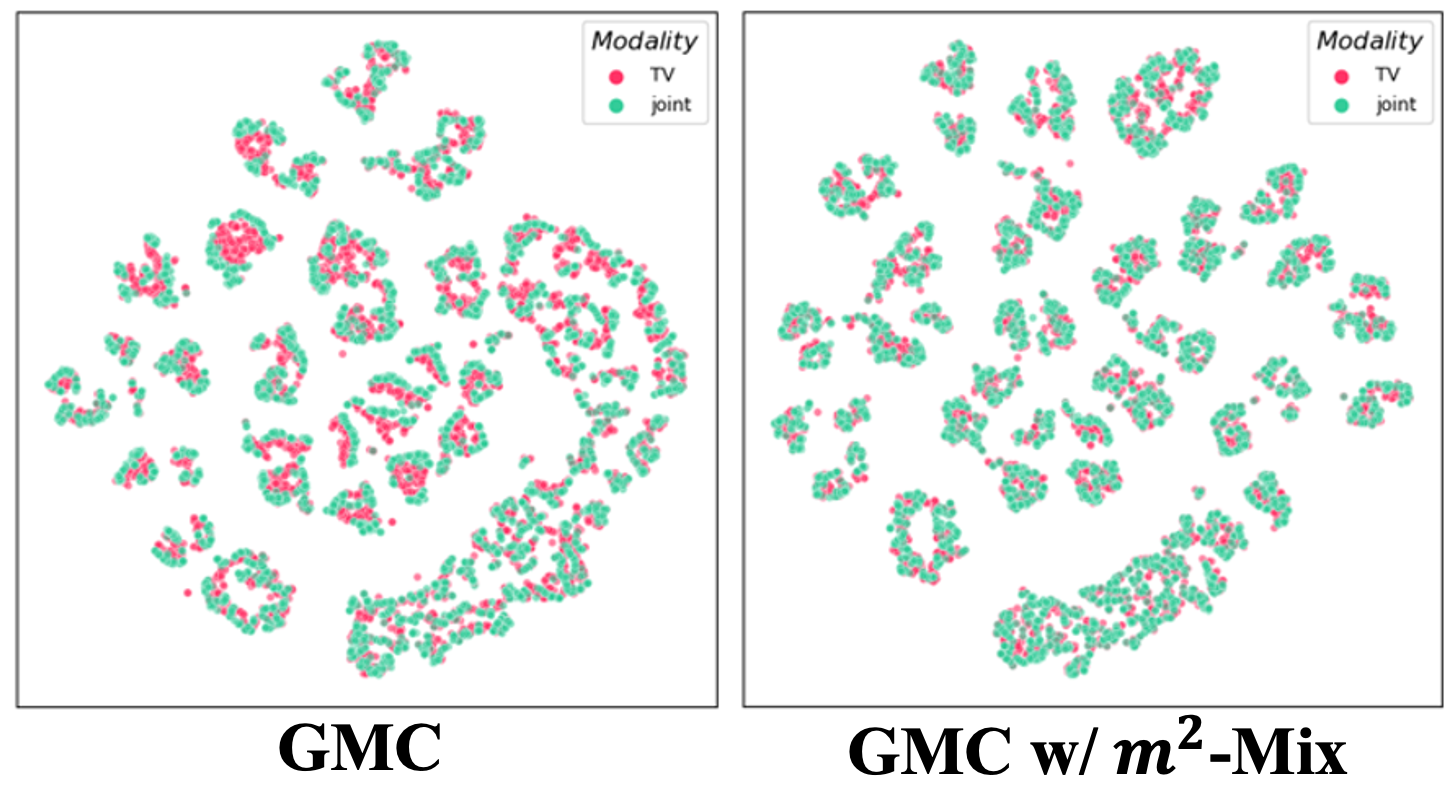}
    \captionof{figure}{t-SNE~\cite{van2008visualizing} on CMU-MOSEI with partial information. Each color denotes embeddings of partial and joint modality.}
    \label{fig:tsne_gmc}
    \vspace{-1em}
\end{wraptable}
Different from the CLIP fine-tuning cases, we use the multi-modal mixed representation as both positive and negative pairs with the target joint representation because our goal in this task is to align the embedding between the joint and other modalities.
As shown in Tab. \ref{tab:results_gmc_acc}, $m^2$-Mix coherently improves the performance of GMC in terms of accuracy, alignment, and uniformity. While GMC strongly aligns the embedding of partial and joint modality, the alignment is further enhanced by the aid of $m^2$-Mix (also confirmed in Fig. \ref{fig:tsne_gmc}), which results in superior performance when only partial information is given in test-time (modality missing). These results justify the use of $m^2$-Mix for robust learning under modality-incomplete scenarios.

\begin{table*}[thpb!]
\centering
\caption{Accuracy (acc.)($\uparrow$), alignment (align.)($\uparrow$), and uniformity (unif.)($\uparrow$) of multi-modal learning methods on CMU-MOSEI under complete and partial modalities. Averaged performance of five runs.}
\label{tab:results_gmc_acc}
\begin{adjustbox}{width=\textwidth}
\begin{tabular}{@{}l*{19}{c|}*{3}{c}@{}}
\toprule
 \multirow{2}{*}{} & \multicolumn{20}{c}{Test-time Observed Modalities} \\ \cmidrule(l){1-21} 
 & \multicolumn{2}{c|}{Full (T+V+A)} & \multicolumn{3}{c|}{T} & \multicolumn{3}{c|}{V} & \multicolumn{3}{c|}{A} & \multicolumn{3}{c|}{T+V} & \multicolumn{3}{c|}{T+A} & \multicolumn{3}{c}{V+A} \\ \cmidrule(l){2-21}
 & acc. & unif. & acc. & align. & unif.& acc. & align. & unif.& acc. & align. & unif.& acc. & align. & unif.& acc. & align. & unif.& acc. & align. & unif. \\ \cmidrule(l){1-21}
 
MulT~\cite{tsai2019multimodal} & 80.5 & 0.99 & 60.0 & - & 1.03 & 53.9 & - & 2.07 & 52.7 & - & 0.62 & 57.8 & - & 1.27 & 58.8 & - & 0.77 & 54.6 & - & 1.36\\
GMC~\cite{poklukar2022gmc} & 80.1 & 3.06 & 78.5 & 0.20 & 3.03 & \textbf{64.7} & 0.17 & 3.01 & 66.0 & 0.09 & 3.03 & 77.0 & 0.07 & 2.94 & 77.4 & 0.08 & 3.00 & 67.3 & 0.05 & 2.98 \\
GMC+$m^{2}$-Mix & \textbf{80.5} & \textbf{3.18} & \textbf{78.9} & \textbf{0.23} & \textbf{3.17} & 64.2 & \textbf{0.19} & \textbf{3.15} & \textbf{66.2} & \textbf{0.12} & \textbf{3.15} & \textbf{77.8} & \textbf{0.08} & \textbf{3.08} & \textbf{77.9} & \textbf{0.09} & \textbf{3.08} & \textbf{67.4} & \textbf{0.06} & \textbf{3.10} \\ \bottomrule
\end{tabular}
\end{adjustbox}
\end{table*}

\subsection{Multi-Modal Embedding Arithmetic}
\label{sec:results_simat}
We expect that well-learned multi-modal embedding represents the structural relationship between instances like word vectors~\cite{ethayarajh2018towards}. For validation, we evaluate the learned embeddings on SIMAT~\cite{couairon2021embedding}. SIMAT evaluates text-driven image representation by retrieving a new image with the highest similarity to the latent vector $x$, which is transformed by text \textit{delta vectors} when we change the word in an original text, i.e., formulated as: $x = I_{original} + \lambda\cdot (T_{new}-T_{original})$.
Here, $I$ and $T$ are image and text embedding vectors, and $\lambda$ is a hyper-parameter about the strength of transformation. Table \ref{tab:simat} presents the quantitative scores across methods after fine-tuning on Flickr30k and MS COCO, and evaluated on SIMAT. Learned representation from $m^{3}$-Mix shows stable scores on both a multi-modal model and jointly fine-tuned uni-modal models, which is further confirmed by qualitative experiments (Fig. \ref{fig:simat_vis}). These support that $m^{3}$-Mix can be adopted as a delicate fine-tuner when embedding geometric structure and arithmetic property should be considered, e.g., controllable generation~\cite{li2023gligen}.

\begin{minipage}{\textwidth}
\begin{minipage}[hpt]{0.46\textwidth}
\centering
\captionof{table}{$m^{3}$-Mix shows the stable SIMAT Score ($\uparrow$) on CLIP and joint fine-tuning of uni-modal pre-trained models.}
\label{tab:simat}
\tiny
\begin{tabular}{lccc}
\toprule
& CLIP & CLIP & BERT + RN50                                                   \\
& (MS COCO) & (Flickr30k) & (Flickr30k)                                            \\ \hline
ZS  & 34.5           & 34.4             & 6.1                                      \\ 
ES~\cite{liang2022mind} & 34.6           & 34.5             & 6.2                                      \\ \hline
FT  & 42.3           & \textbf{40.8}             & 15.4                                     \\
$i$-Mix~\cite{lee2021imix} & 37.3           & 40.0             & 12.9                         \\
Un-Mix~\cite{shen2022unmix} & 42.9           & 38.5             & 15.8                         \\ \hline
$m^{3}$-Mix     & \textbf{44.4}  & 38.9             & \textbf{19.0}                \\ \bottomrule
\end{tabular}
\end{minipage}
\begin{minipage}[hpt]{0.48\textwidth}
  \centering
    \includegraphics[width=\textwidth]{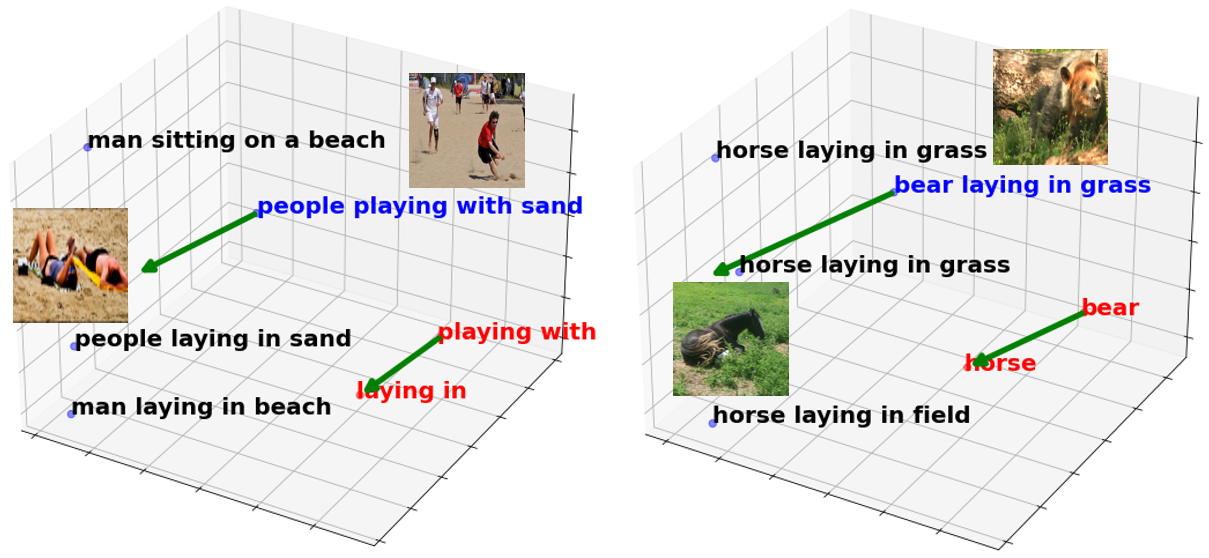}
    \captionof{figure}{Embeddings of $m^3$-Mix on SIMAT source-target texts (reds) and images (source and top-1 retrieved). Black texts correspond to top-3 retrieved images from a source image-caption.} \label{fig:simat_vis}
\end{minipage}
\end{minipage}

\subsection{Multi-Modal Mixup on a State-of-the-Art Vision-Language Model}
\ctn{In this section, we investigate whether our multi-modal Mixup can also be beneficial for improving other recent large-scale vision-language models beyond CLIP. For this, we adopt Contrastive Captioner (CoCa)~\cite{yu2022coca} ViT-L/14 configuration that pre-trained on LAION-2B~\cite{schuhmannlaion} from \texttt{OpenCLIP} library as our target backbone model. We consider three learning objectives for CoCa fine-tuning: (1) autoregressive captioning loss (Cap), (2) contrastive loss and captioning loss (CL $+$ Cap), and (3) contrastive loss, $\bL_{m^{2}\text{-Mix}}$, and captioning loss (CL w/ $\bL_{m^{2}\text{-Mix}} +$ Cap). For all three methods, we train the model on MS COCO over one epoch with \texttt{OpenCLIP}-provided hyperparameter configuration. After that, we evaluate each model for image captioning (Tab. \ref{tab:coca_captioning}) on MS COCO and zero-shot and fine-tuned cross-modal retrieval (Tab. \ref{tab:coca_retrieval}) on Flickr30K and MS COCO.}

\begin{table}[thpb]
\centering
\caption{\ctn{Image captioning results on MS COCO with CoCa ViT-L/14 model. We fine-tune the CoCa on MS COCO for 1 epoch with three different learning objectives and evaluate them in terms of five conventional metrics. $m^2$-Mix achieves performance gain on all the five metrics.}}
\vspace{1em}
\small
\begin{tabular}{@{}lccccc@{}}
\toprule
\multicolumn{1}{c}{Method}                 & \multicolumn{4}{c}{Metrics}          \\
              & BLEU@4           & METEOR          & ROUGE-L        & CIDEr           & SPICE           \\ \midrule
ZS           & 7.2          & 12.4          & 26.3          & 35.2          & 9.3          \\ 
Cap           & 36.0          & 29.4          & 57.3          & 125.1          & 23.1          \\
CL + Cap      & 35.7          & 29.3          & 57.1          & 124.9          & 23.0          \\
CL w/ $\bL_{m^{2}\text{-Mix}}$ + Cap & \textbf{36.3} & \textbf{29.5} & \textbf{57.5} & \textbf{125.6} & \textbf{23.2} \\
\bottomrule
\end{tabular} \label{tab:coca_captioning}
\vspace{-1em}
\end{table}
\ctn{In Tab. \ref{tab:coca_captioning}, while a combination of vanilla contrastive loss with captioning loss underperforms the captioning-loss-only training, $m^2$-Mix-assisted contrastive learning further increases the performance on image captioning in terms of five conventional metrics. Besides, Tab. \ref{tab:coca_retrieval} shows that $m^2$-Mix generally improves the retrieval recalls of CoCa on zero-shot and fine-tuned settings. This implies that representation learning with our multi-modal Mixup is beneficial to generative tasks as well as discriminative tasks. The consistent improvement shown in the image captioning task is accorded with observations from other recent works~\cite{fang2022stemm, cheng2023mixspeech} that reveal the effectiveness of cross-modal Mixup on generative tasks by increasing cross-modal alignment.}
\begin{table}[thpb]
\vspace{-1em}
\centering
\caption{Cross-modal retrieval top-1/5 recalls on MS COCO (fine-tuned) and Flickr30k (zero-shot transfer). $m^2$-Mix generally enhances image-text retrieval of a SOTA vision-language model, CoCa.}
\vspace{1em}
\small
\begin{tabular}{@{}lcccc@{}}
\toprule
\multicolumn{1}{c}{Method}  & \multicolumn{4}{c}{(Zero-shot) Flickr30k}               \\
              & i $\rightarrow$ t (R1)       & i $\rightarrow$ t (R5)       & t $\rightarrow$ i (R1)       & t $\rightarrow$ i (R5)       \\ \midrule 
Cap           & 90.4          & 98.5          & 78.5          & 94.4          \\
CL + Cap      & 92.4          & \textbf{99.1} & 79.2          & 94.9          \\
CL w/ $\bL_{m^{2}\text{-Mix}}$ + Cap & \textbf{92.8} & 99.0          & \textbf{79.5} & \textbf{95.1} \\ \midrule
\multicolumn{1}{c}{Method}  & \multicolumn{4}{c}{(Fine-tuned) MS COCO}                   \\
              & i $\rightarrow$ t (R1)       & i $\rightarrow$ t (R5)       & t $\rightarrow$ i (R1)       & t $\rightarrow$ i (R5)       \\ \midrule 
Cap           & 68.5          & 87.9          & 53.2          & 77.8          \\
CL + Cap      & \textbf{73.9} & \textbf{91.2} & 56.3          & 80.4          \\
CL w/ $\bL_{m^{2}\text{-Mix}}$ + Cap & \textbf{73.9}    & 91.0          & \textbf{56.5} & \textbf{80.6} \\
\bottomrule
\end{tabular} \label{tab:coca_retrieval}
\end{table}
\vspace{-1em}
\section{Conclusion}
This paper analyzes the representation learned by a multi-modal contrastive learner. We found that CLIP has separated text-versus-image embedding space \sty{with poor uniformity-alignment.} These polarized embeddings with huge unexploited space may limit the transferability and robustness of representation on downstream tasks. From our findings, we propose \textit{Geodesic Multi-Modal Mixup} that generates hard negatives for robust contrastive learning by mixing two heterogeneous embeddings. Theoretically, we validate that our method produces hardness-guaranteed samples and has desirable asymptotic behavior to induce better generalizable representation. Empirically, the proposed method effectively improves performances on diverse \sty{tasks and perspectives}: retrieval, calibration, few-shot classification under distribution shift, embedding arithmetic, \ctn{and image captioning}.

Though the increased uniformity and alignment of multi-modal representation largely empowers the model to adapt robustly to a diverse range of tasks, we found that reckless uplift of them is harmful in some cases of retrieval (modest increment of uniformity-alignment was better than huge increment of them). Thus, more research on the reliable evaluation of multi-modal representation should be pursued in the era of foundation models.

\paragraph{Acknowledgement} This work was supported by the National Research Foundation of Korea (NRF) grant funded by the Korea government (MSIT) (No.2021R1F1A1060117 and No.2022R1A4A3033874), and also supported by a grant (22183MFDS431) from Ministry of Food and Drug Safety in 2023.

{
\small
\bibliographystyle{unsrt} 
\bibliography{main}
}

\newpage
\appendix

\section{Experiment Setup} \label{appendix:exp_setting}
We first elaborate on the setups for all the experiments, including retrieval and embedding arithmetic, uni-modal classification, multi-modal classification, and Contrastive Captioner (CoCa) image captioning and retrieval, in each separate section.

\subsection{Cross-modal Retrieval and SIMAT}
\paragraph{Dataset} Here, we hold the Flickr30k and MS COCO, two representative vision-language benchmark datasets. Flickr30k contains 30K image-text pairs as a train split\footnote{https://www.kaggle.com/hsankesara/flickr-image-dataset}, 1k for validation and test splits\footnote{https://github.com/BryanPlummer/flickr30k\_entities}. 
For MS COCO, we adopt the 2017 version of it from the COCO Database\footnote{https://cocodataset.org/\#download}. MS COCO contains 118k image-text pairs for train split and 5k for both validation and test splits. When there are multiple captions for one image, we always select the first caption to construct an image-text pair.
To validate the multi-modal embedding arithmetic, we use the SIMAT dataset~\cite{couairon2021embedding}. SIMAT is a benchmark created for evaluating the text-driven image transformation performance of multi-modal embedding. It contains 6k images, 18k transformation queries that have pairs of (source word, target word, source image, target image), and 645 captions constructed with subject-relation-object triplets that have at least two corresponding images. The goal of SIMAT task is to retrieve an image, which is well-modified by a specific text transform to match with the ground truth transform target images. 

\paragraph{Model Description} For retrieval and embedding arithmetic tasks, we adopt CLIP ViT-B/32 checkpoint of OpenAI official lease\footnote{https://github.com/openai/CLIP} as our backbone model. For cross-modal retrieval with disjointly pre-trained uni-modal models, we utilize ResNet-50~\cite{he2016deep} with a pre-trained checkpoint of \texttt{torchvision} as an image encoder and \texttt{BERT-base-uncased} from HuggingFace as a text encoder. To match the dimensions of these two uni-modal models, we add a projection head on top of each encoder, respectively.

\paragraph{Baseline Methods} First, we consider the zero-shot inference of CLIP (ZS)~\cite{radford2021learning} as a strong baseline (in the case of retrieval with uni-modal pre-trained models, we just project the image and text embeddings to shared vector space with randomly initialized matrix, and perform similarity-based inference as ZS.), and embedding shift (ES)~\cite{liang2022mind} which computes a delta vector (difference between mean vectors of image and text embeddings) and then manually modifies the modality gap along with delta vector direction without explicit training. Then, a vanilla fine-tuning (FT) with standard contrastive loss (Eq. 1 of main paper) and its higher-temperature variants ($\tau=\{0.05, 0.01\}$) are considered. Additionally, we take account of two uni-modal mixup-based contrastive learning methods $i$-Mix~\cite{lee2021imix} and Un-Mix~\cite{shen2022unmix} those mix images in the input space. While the original implementation of $i$-Mix takes a randomly sampled image as a mixture component, we take a flipped batch sample as a mixture component for computational efficiency like as Un-Mix. So the only difference between $i$-Mix and Un-Mix is whether we construct the final objective as a sum of normal and mixed sample contrastive loss~\cite{shen2022unmix} or sorely mixed sample contrastive loss (\cite{lee2021imix}).

\paragraph{Metric} As a standard metric for retrieval tasks, we report top-1 recall (R1) and top-5 recall (R5) on both image-to-text and text-to-image directions. For SIMAT task, following the original paper~\cite{couairon2021embedding}, we performed the OSCAR-based evaluation and reported the SIMAT score in the original paper. It measures the similarity between the transformed image and text captions via OSCAR framework~\cite{li2020oscar}. 

\paragraph{Implementation Detail} We fine-tune CLIP with \texttt{eval()} mode stable training\footnote{https://github.com/openai/CLIP/issues/150/} and under \texttt{FP16} precision for computational efficiency. On both Flickr30k and MS COCO, we train each method over 9 epochs with batch size 128 via Adam optimizer ($\beta_{1}=0.9$, $\beta_{2}=0.98$, and $\epsilon=1e-6$). As shared hyperparameters, we search for the best initial learning rate from \{1e-6, 3e-6, 5e-6, 7e-6, 1e-5\} and weight decay from \{1e-2, 2e-2, 5e-2, 1e-1, 2e-1\} for all training methods (Initial learning rate is decayed in each epoch by the exponential scheduler with decaying parameter 0.9). To construct our complete objective $m^{3}\text{-Mix}$, we weighing the $L_{\text{CLIP}}$ and $L_{m^{2}\text{-Mix}}$ and uni-modal geodesic Mixup variants ($L_{V/L/VL\text{-Mix}}$). Specifically, we pivot the weight of $L_{\text{CLIP}}$ as 1.0 and sweep the weighting coefficient of other loss components for each dataset generally from \{0.0, 0.01, 0.1, 0.2, 0.3, 0.5\}\footnote{We scheduled the strength of sum of the Mixup-based loss terms by \texttt{L\_mix$/$epoch}}. The parameter $\alpha$ of Beta distribution $Beta(\alpha, \alpha)$ that determines the mixture ratio is set to 0.5 for the multi-modal Mixup and 2.0 for uni-modal Mixups. For embedding shift (ES)~\cite{liang2022mind}, we sweep $\lambda$ from -0.1 to 0.1 by 0.01 and report the best results among them. While the search range of ES from official implementation is from -2.5 to 2.5 by 0.125, we observe the finer search range gives better results. 

\subsection{Uni-modal Classification}
\paragraph{Dataset} We consider three common transfer learning benchmark datasets, OxfordPets~\cite{6248092}, SVHN~\cite{37648}, and CLEVR~\cite{johnson2017clevr}, to validate the general few-shot adaptation capability. For evaluation of robustness on distribution shift, we consider the ImageNet-1k as a source dataset (models are adapted to) and ImageNetV2~\cite{recht2019imagenet}, ImageNet-A~\cite{hendrycks2021natural}, ImageNet-R~\cite{hendrycks2021many}, and ImageNet-Sketch~\cite{wang2019learning} as target evaluation datasets those are considered as different kinds of natural distribution shift from ImageNet.

\paragraph{Model Description} For uni-modal few/zero-shot classifications, we also adopt CLIP ViT-B/32 as the default backbone for ours and baseline fine-tuning methods in our manuscript and also evaluate CLIP ViT-B/16 and CyCLIP ResNet50 in Section \ref{appendix:results} of this Supplementary Material.

\paragraph{Baseline Methods} As standard baselines, we first consider zero-shot CLIP (ZS) and vanilla fine-tuning (FT) with contrastive loss. Then, we perform exhaustive ablation ($V$-, $L$-, $VL$-, and $m^{2}$-Mix) as well as our complete objective $m^{3}$-Mix with its high-temperature variant. To further compare our approach with state-of-the-art fine-tuning methods, we consider MaPLe~\cite{khattak2022maple} that optimizes the continuous prompts inside the text and image encoders of CLIP, and the contrastive loss extended version of uni-modal fine-tune methods: LP-FT~\cite{kumar2022fine} which trains classification head and full modal separately in a two-stage manner, and WiSE-FT~\cite{wortsman2021robust} which performs parameter-space ensemble between the pre-trained checkpoint and fine-tuned checkpoint. Additionally, we consider the ES and our $m^{3}\text{-Mix}$ as the \textit{plug-in} methods to improve the above three state-of-the-art fine-tuning methods that are denoted as method names w/ ES or $m^3\text{-Mix}$ in Tab. 4 and 5 of the main paper.

\paragraph{Metric} For both the few-shot adaptation and distribution shift setting, we report top-1 accuracy as the In-Distribution Accuracy (ID Acc.) and Out-Of-Distribution Accuracy (OOD Acc.), respectively. 

\paragraph{Implementation Detail} In this paper, we propose new contrastive losses $m^{2}\text{-Mix}$ and $m^{3}\text{-Mix}$ which consume the image-text paired instances. However, the above datasets provide class name labels only and do not have captions corresponding to each image. To make CL methods amendable for this setting, we adopt a common prompt \texttt{'a photo of {classname}'} that wraps the class name with a short context and use this as captions of images. Different from image-caption-based contrastive learning on Flickr30k and MS COCO, a batch of ImageNet-1K contains multiple samples that are assigned to the same class. We construct the label map for contrastive loss by regarding all of the samples from a class as positives. Following~\cite{zhou2022conditional, khattak2022maple}, we perform the tasks under the same few-shot evaluation protocol: 16-shot training samples per class and inference on the entire test set. 
To construct the contrastive loss, we first compute the pivot classifier embedding by forwarding all possible class category names to the text encoder. Then, we calculate the pairwise similarity between in-batch image embedding and pivot embedding and construct the label matrix by reflecting the fact that there are many positive images corresponding to a text embedding (for each class). To implement the contrastive loss with multi-modal Mixup, we mix the in-batch image embedding and text embedding and contrast the resulting mixed embedding with image and pivot embedding, respectively.

About training configuration, in the distribution shift setting, we train all methods (except MaPLe) on 20 epochs with batchsize 100 via AdamW optimizer with default parameters. Due to MaPLe's huge memory requirements, we set the batchsize to 4 and train a single epoch. As shared hyperparameters, we pivot the initial learning rate to 1e-6 and search for the best maximum learning rate from \{1e-6, 3e-6, 5e-6, 7e-6, 1e-5\} and weight decay from \{0, 1e-3, 5e-3, 1e-2, 5e-2 1e-1\} for all training methods (except MaPLe's learning rate sweep from \{5e-3, 1e-3, 5e-4, 1e-4\}). Here we use the one-cycle cosine learning rate scheduler. For the few-shot adaptation in a general setting, we train each method over 200 epochs (40 epochs for MaPLe) with the same batchsize, optimizer, and hyperparameter sweep range. In both two settings, we use the same data augmentation procedure (random resize crop and random flip)~\cite{zhou2022conditional, khattak2022maple} for all methods. Note that for LP-FT, we train both the linear head and full models during half of the entire epochs, and we do not use data augmentation in the linear head training phase following the authors' proposal. For our methods, weighting coefficient of $m^2$-Mix and uni-modal mixups are explored over \{0.01, 0.1, 0.2, 0.3, 0.4, 0.5\}, and the parameters of Beta distribution are swept over \{0.2, 0.5\}.

\subsection{Multi-modal Classification}
\paragraph{Dataset} To evaluate the multi-modal representation learning under video emotional classification, we consider the CMU-MOSEI~\cite{7742221}, a popular benchmark for multi-modal sentiment analysis. CMU-MOSEI consists of three modalities textual (T), visual (V), and audio (A), and contains 23,453 YouTube video clips about diverse movie reviews, and each clip is annotated with ordinal labels ranging from -3 (strong negative) to 3 (strong positive). In the training phase, three modalities are fully available to all methods, and only one or two modalities are given in the evaluation phase to measure the robustness under modality missing as well as the informativeness of individual-modality representations.

\paragraph{Model Description and Baseline Methods} To construct backbone architecture, following Poklukar et al.~\cite{poklukar2022gmc}, we adopt the Multimodal Transformer (MulT)~\cite{tsai2019multimodal} as a joint-modal encoder which enables the commutation among modalities with cross-modal attention block. To enhance the explicit alignment between modalities, Poklukar et al.~\cite{poklukar2022gmc} propose Geometric Multimodal Contrastive Learning (GMC). In addition to the joint encoder. GMC introduces lightweight modality-specific encoders constructed by a single Gated Recurrent Unit~\cite{cho2014learning} followed by a linear projection layer, then performs contrastive learning between joint representation (from a joint encoder) and uni-modal representation (from modality-specific encoders). We set the MulT as a standard baseline, GMC as a contrastive learning-enhanced baseline, and then plug our $m^{2}\text{-Mix}$ to GMC objective to validate whether our method can give additional benefits to multi-modal representation learning. 

While MulT learns the joint encoder only with standard classification loss (i.e., cross-entropy loss; $L_{\text{ce}}$), GMC learns joint and modality-specific encoders with the objective function $L_{\text{ce}}+L_{\text{GMC}}$ where $L_{\text{GMC}}$ deals with the sum of all one-to-joint contrastive losses. On top of GMC, $L_{m^2\text{-Mix}}$ is integrated with a trade-off hyperparameter $\beta$: $L_{\text{ce}}+L_{\text{GMC}}+\beta L_{m^2\text{-Mix}}$.

\paragraph{Metric} As mentioned earlier, in the inference time, each method can encounter partial modalities among T, V, and A. To this end, we evaluate each method under 7 environments sorted by the available modalities: (T), (V), (A), (T,V), (T,A), (V,A), (T,V,A). Then, we measure the classification accuracy, F1-score (in supplementary material), uniformity, and alignment.

\paragraph{Implementation Detail} We train all methods over 40 epochs with batchsize 128 via Adam optimizer with the default configuration. Following~\cite{poklukar2022gmc}, we set the learning rate to 1e-3 and do not apply the weight decay. The trade-off parameter $\beta$ and the parameter $\alpha$ of Beta distribution \texttt{Beta($\alpha$, $\alpha$)} are optimized among \{0.1, 0.2, 0.3, 0.4, 0.5\} and \{0.5, 1.0, 1.5, 2.0\}, respectively. Those are selected $\beta=0.2$ and $\alpha=2.0$.
To implement the contrastive loss with $m^2\text{-Mix}$, we randomly sample two modalities for each training epoch and mix them to build a mixed representation. Then, we compute (1) the positive similarity between paired mixed and joint representation and (2) the negative similarity between non-paired mixed and joint representation. Finally, we compute the modified $L_{m^2\text{-Mix}}$ as a negative logarithm of the sum of positive similarities over the sum of negative similarities.
Different from the CLIP fine-tuning cases, we use the multi-modal mixed representation as both positive and negative pairs with the target joint representation because our goal in this task is to align the embedding between the joint and other modalities.
For inference with partial modalities, we average the representations from given modalities to make a single embedding that is sent to the classifier head for a fair comparison with~\cite{tsai2019multimodal, poklukar2022gmc}.

\subsection{Multi-modal Mixup for Contrastive Captioner}
\paragraph{Dataset} To demonstrate the effectiveness of the multi-modal Mixup on a state-of-the-art vision-language model, Contrastive Captioner (CoCa), we perform cross-modal retrieval on Flickr30k and MS COCO, and image captioning task on MS COCO (that of 2014). 

\paragraph{Model Description and Baseline Methods} We adopt LAION-2B pre-trained CoCa ViT-L/14 from \texttt{OpenCLIP} library as our backbone model, and consider three learning objectives for CoCa fine-tuning: (1) autoregressive captioning loss (Cap), (2) contrastive loss and captioning loss (CL $+$ Cap), and (3) contrastive loss, $\bL_{m^{2}\text{-Mix}}$, and captioning loss (CL w/ $\bL_{m^{2}\text{-Mix}} +$ Cap).

\paragraph{Metric} For image-text retrieval, we adopt top-1 and top-5 recalls likewise CLIP retrieval setup. For image captioning, five standard metrics: BLEU-4~\cite{10.3115/1073083.1073135}, METEOR~\cite{banerjee-lavie-2005-meteor}, ROUGE-L~\cite{lin-2004-rouge}, CIDEr~\cite{vedantam2015cider}, and SPICE~\cite{anderson2016spice} are evaluated.

\paragraph{Implementation Detail} For all three methods, we train the model on MS COCO over one epoch with \texttt{OpenCLIP}\footnote{https://github.com/mlfoundations/open\_clip}-provided hyperparameter configuration, i.e., 128 as batch size, 1e-5 as learning rate, 0.1 as weight decay, and 1000 as learning rate warm-up steps. After fine-tuning on MS COCO, we evaluate the model on MS COCO for fine-tuned image-text retrieval and image captioning, and on Flickr30k for zero-shot transferred image-text retrieval. Here, we adopt \texttt{CLIP\_benchmark}\footnote{https://github.com/LAION-AI/CLIP\_benchmark} library for easy evaluation. For training of CL $+$ Cap, we weight $\bL_{CL}$ as 1.0 and $\bL_{Cap}$ as 2.0. Our $m^2\text{-Mix}$ related parameters are explored over \{0.1, 0.2, 0.3, 0.4, 0.5, 1.0\} for Beta distribution parameter and \{0.1, 0.2, 0.25, 0.3, 0.35, 0.4, 0.5, 0.7, 1.0\} for $\bL_{m^2\text{-Mix}}$ weighting coefficient.

\newpage

\section{Pseudo Code} \label{appendix:pseudo}

\begin{algorithm2e}
\SetAlgoLined
    \PyComment{X,Y : image batch, text batch}\\ 
    \PyComment{f,g : learnable image encoder and text encoder}\\
    \PyComment{t1, t2 : trainable temperature parameters}\\
    \PyComment{alpha1, alpha2 : parameters for Beta Distribution}\\
    \PyComment{args.\{m2mix, vmix, lmix, vlmix\} : weighting parameters}\\
\PyCode{def ce(logits,targets):}\\
\Indp
    \PyCode{return (-targets*nn.LogSoftmax(dim=-1)(logits)).sum()}\\
\Indm
\PyCode{def cross\_entropy\_2D(logits,targets):}\\
\Indp
    \PyCode{return ((ce(logits,targets)+ce(logits.T,targets.T))/2).mean()}\\
\Indm
\PyCode{def geodesic\_mix(lambda,a,b):}\\
\Indp
    \PyCode theta = torch.acos( (a*b).sum(dim=[1])).view(a.shape[0],1)\\
    \PyCode n1 = torch.sin(lambda*theta)/torch.sin(theta)*a\\
    \PyCode n2 = torch.sin((1-lambda)*theta)/torch.sin(theta)*b\\
    \PyCode return n1+n2\\
\Indm
\PyCode{ }\\
\PyCode{def ContrastiveLoss(X,Y,f,g,t1,t2,args)}\\
\Indp
    \PyCode{I = torch.eye(X.shape[0])} \\ 
    \PyCode{I\_R = torch.flip(I,dims=[0])}\\
    \PyCode{I\_X, I\_XD = I+I\_R, 1-(I+I\_R)}\\
    \PyCode{If, Tf = f(X), g(Y)} \PyComment{L2 normalized features} \\ 
    \PyCode{logits = If@Tf.T} \PyComment{Original logit} \\
    \PyCode{loss = cross\_entropy\_2D(logits/t1,I)} \PyComment{2D Cross Entropy}\\
    \PyCode{if args.m2mix:} \\
    \Indp
        \PyCode{lambda = random.betavariate(alpha2,alpha2)} \\
        \PyCode{mix = geodesic\_mix(lambda,If,Tf)} \\
        \PyCode{logits2\_i = mix@Tf.T} \\
        \PyCode{logits2\_i = logits*I + logits2\_i*(1-I)} \\
        \PyCode{logits2\_t = mix@If.T} \\
        \PyCode{logits2\_t = logits.T*I + logits2\_t*(1-I)} \\
        \PyCode{loss += args.m2mix*(ce(logits2\_i/t2,I) + ce(logits2\_t/t2,I))/2} \\
    \Indm
    \PyCode{if args.vmix:}  \\
    \Indp
        \PyCode{lambda = random.betavariate(alpha1,alpha1)} \\
        \PyCode{mix = geodesic\_mix(lambda,If,If.flip())} \\
        \PyCode{logits2 = mix@Tf.T} \\
        \PyCode{logits2 = logits2*I\_X + logits*I\_XD} \\
        \PyCode{loss += args.vmix*cross\_entropy\_2D(logits2/t1,lambda*I + (1-lambda)*I\_R)} \\
    \Indm
    \PyComment{L-Mix is omitted} \\ 
    \PyCode{if args.vlmix:} \\
    \Indp
        \PyCode{lambda = random.betavariate(alpha1,alpha1)} \\
        \PyCode{mix\_I = geodesic\_mix(lambda,If,If.flip())} \\
        \PyCode{mix\_T = geodesic\_mix(lambda,Tf,Tf.flip())} \\
        \PyCode{logits2 = mix\_I@mix\_T.T} \\
        \PyCode{logits2 = logits2*I + logits*(1-I)} \\
        \PyCode{loss += args.vlmix*cross\_entropy\_2D(logits2/t1,I)} \\
    \Indm
    \Indp
    \PyCode{return loss}
\Indm
\caption{PyTorch-style Implementation Code for Geodesic Multi-Modal Mixup}
\label{algo:your-algo}
\end{algorithm2e}

\newpage

\section{Additional Results} \label{appendix:results}
\subsection{Mixed Embedding Analysis}
In Fig. \ref{fig:mix_retrieval}, we post examples of image-retrieval results by our $m^{2}$-Mix. Big images and text on the left denote the original image and text pair. The right top and bottom denote the top-3 retrieved images from the original image embedding and mixed embedding, respectively. Overall, retrieved images by $m^{2}$-Mixed embedding contain more rich semantics that is derived from both image and text.
\begin{figure}[h!]
    \centering
    \includegraphics[width=0.95\textwidth]{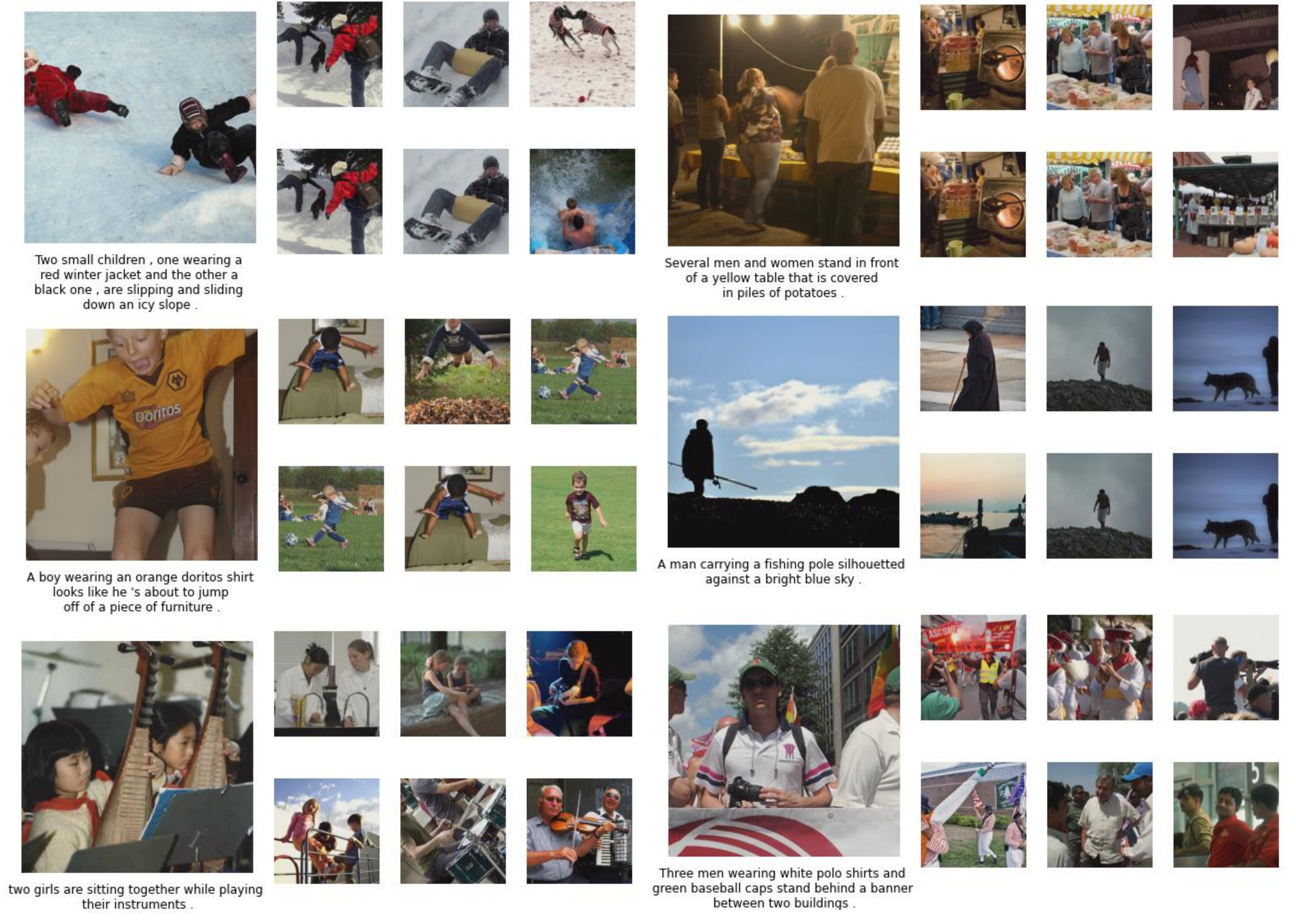}
    \caption{Retrieved images by original image embedding and mixed embedding on Flickr30k.}
    \label{fig:mix_retrieval}
\end{figure}

\subsection{Additional Results on Uni-modal Classification}
This section provides results of 16-shot uni-modal classification on four new datasets (EuroSAT~\cite{helber2019eurosat}, FGVC Aircraft~\cite{maji2013fine}, UCF101~\cite{soomro2012ucf101}, Stanford Cars~\cite{6755945}) with two different models (CLIP ViT-B/16 and CyCLIP~\cite{goel2022cyclip} ResNet50) that are lacking in the main paper. We perform fine-tuning of them from their official checkpoint relase\footnote{https://github.com/openai/CLIP}\footnote{https://github.com/goel-shashank/CyCLIP} with the same hyperparameter sweep range described in \ref{appendix:exp_setting} of Supplementary.
In Table \ref{appendix:add_uniclf}, our $m^2$-Mix brings consistent performance gain across all datasets and models with some significant boosting on the FGVC Aircraft and Stanford Cars datasets. Thus, $m^2$-Mix is a general approach that can enhance the representation learning on various settings.

\begin{table}[bthp]
\centering
\caption{Few-shot classification results with CLIP-VIT-B/16 and CyCLIP-RN50 models. We follow the same few-shot evaluation protocol and contrastive learning strategy with Sec 5.4. in our manuscript. $m^2$-Mix consistently outperforms the baseline methods across four datasets. Especially, on Aircraft, $m^2$-Mix achieves 8.5\% and 4.7\% gain over FT in CLIP and CyCLIP, respectively, and 5.2\% gain over FT in CyCLIP on Cars.}
\vspace{1.25em}
\begin{tabular}{ll|cccc}
\toprule
\multicolumn{1}{c}{Model} & \multicolumn{1}{c}{Method}                 & \multicolumn{4}{|c}{Dataset}                                                                                \\ 
                          &                                            & \multicolumn{1}{|l}{EuroSAT} & \multicolumn{1}{l}{Aircraft} & \multicolumn{1}{l}{UCF101} & \multicolumn{1}{l}{Cars} \\ \midrule
CLIP (ViT-B/16)           & ZS                                         & 48.41                    & 24.81                   & 67.46                      & 65.33                    \\
                          & FT                                         & 94.03                    & 60.61                   & 86.36                      & 88.58                    \\
                          & FT w/ $m^{2}$-Mix           & \textbf{94.33}           & \textbf{69.07}          & \textbf{86.94}                      & \textbf{90.36}           \\ \midrule
CyCLIP (RN50)            & FT                                         & 84.98                    & 48.19                   & 67.25                      & 67.02                    \\
                          & FT w/ $m^{2}$-Mix           & \textbf{85.22}           & \textbf{52.96}          & \textbf{68.97}             & \textbf{72.22}           \\ \bottomrule                         
\end{tabular} \label{appendix:add_uniclf}
\end{table}

Next, in Figure \ref{fig:test2} and \ref{fig:test3}, we perform ablation on the parameter $\alpha$ of Beta distribution, which stochastically determines the mixing ratio between two modalities. Red and Blue colors denote a constant parameter and a linear scheduling parameter, respectively. We see that lower value $\alpha$ (U-shaped Beta distribution) generally achieves better performance than larger values (uniform or reversed U-shaped) on the two classification datasets.

Linear scheduling of Beta parameters drives promising results in some cases, e.g., $1.0 -> 0.1$ and $2.0 -> 0.1$ in Stanford Cars. It seems crucial to enforce that the shape of Beta distribution ends up with a U-shape for the success of scheduling variants. That is, the small-to-many mixing fashion is better than that of half-to-half for the geodesic multi-modal Mixup on classification.

\begin{figure}[h!]
    \centering
    \includegraphics[width=0.72\textwidth]{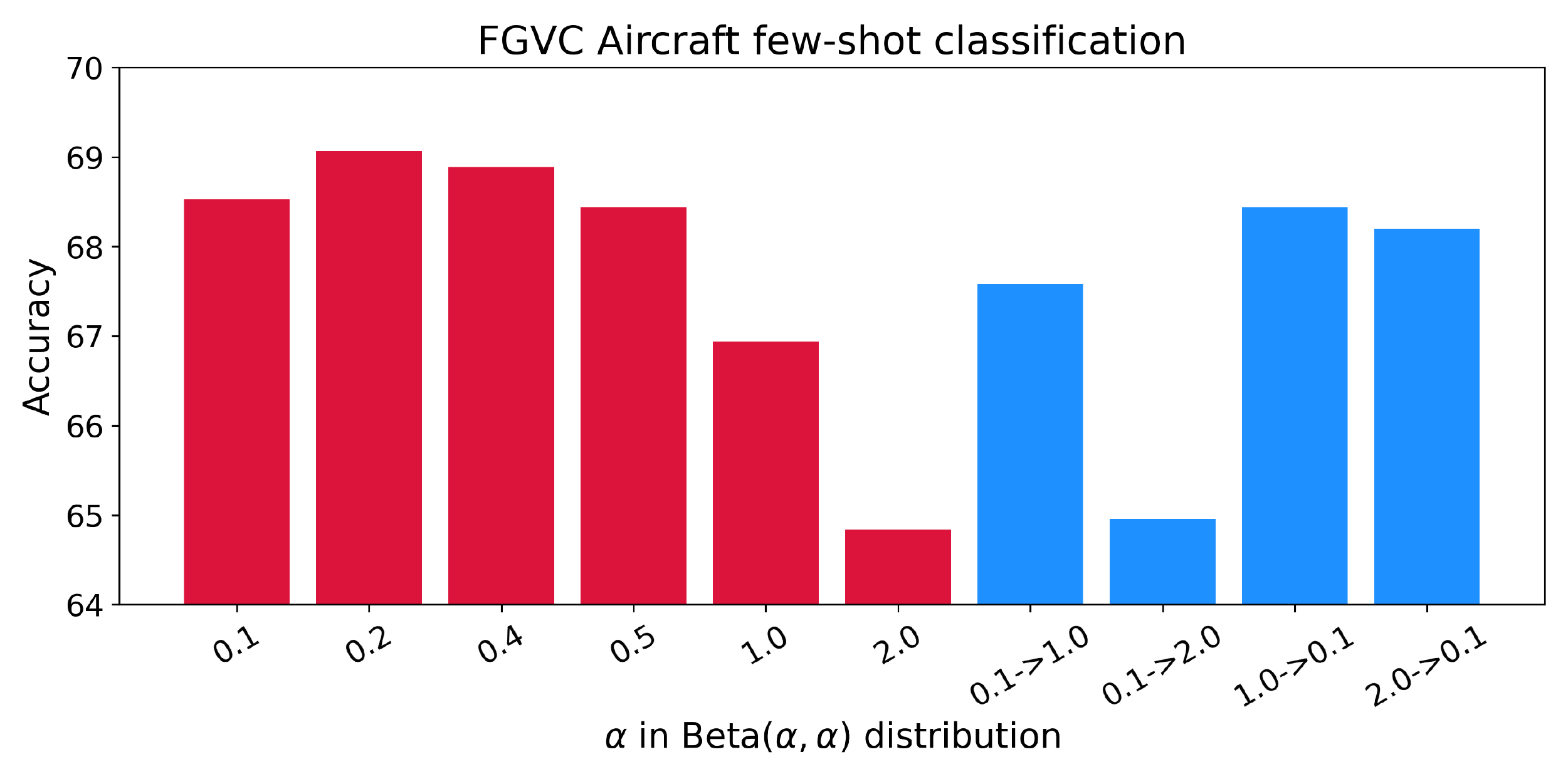}
    \caption{Few-shot classification results on FGVC Aircraft dataset. We varied the parameter of Beta distribution, constant (red) or scheduled (blue), to simulate diverse situations of mixed samples. It always achieves better performance than FT (60.61), even under varying parameters.}
    \label{fig:test2}
\end{figure}
\begin{figure}[h!]
\centering
    \includegraphics[width=0.72\textwidth]{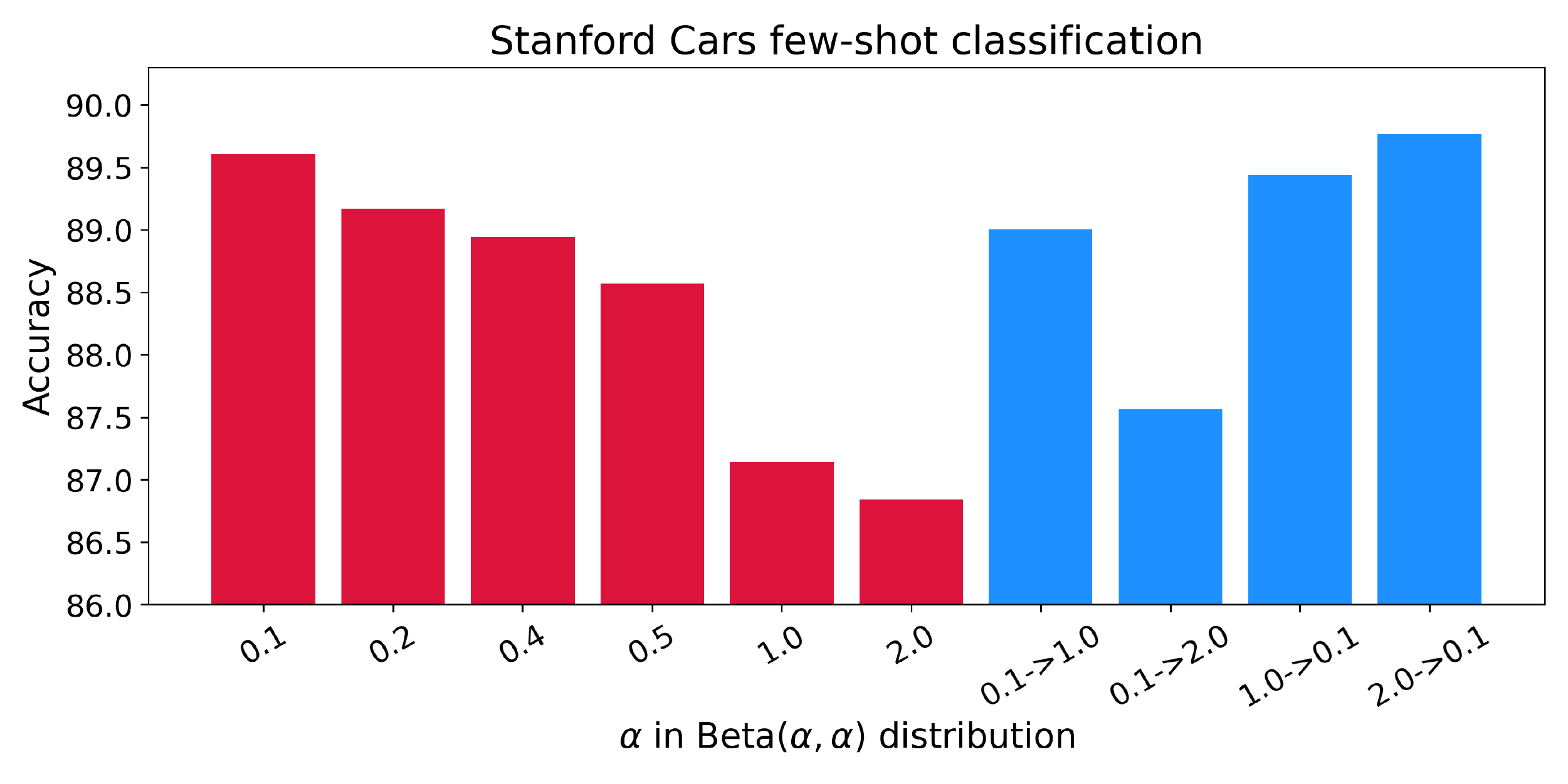}
    \caption{Few-shot classification results on Stanford Cars dataset. We varied the parameter of Beta distribution, constant (red) or scheduled (blue), to simulate diverse situations of mixed samples. It generally achieves better performance than FT (88.58), even under varying parameters.}
    \label{fig:test3}
\end{figure}

\subsection{Additional Results on Multi-modal Classification}
In addition to classification accuracy (in the main paper), we additionally present the F1-score (in Tab. \ref{tab:results_gmc_f1}) for diagnosis on classification results. While GMC with $m^2$-Mix is outperformed by GMC in two of three single-modality cases, it shows superior results on two-modality-given cases based on explicit enforcement of bi-to-joint alignment during training.

\begin{table*}[h!]
\caption{Classification F1-score on CMU-MOSEI under complete and partial observation modalities. We report the mean performance and standard deviation of five runs.}
\begin{adjustbox}{width=\textwidth}
\begin{tabular}{@{}lccccccc@{}}
\toprule
 \multirow{2}{*}{Method} & \multicolumn{7}{c}{Test-time Observed Modalities} \\ \cmidrule(l){2-8} 
 & Full(T+V+A) & T & V & A & T+V & T+A & V+A \\ \cmidrule(l){2-8} 
MulT & 0.8056$\pm$0.004& 0.6909$\pm$0.051 & 0.5678$\pm$0.107 & 0.6021$\pm$0.151 & 0.6453$\pm$0.096 & 0.6657$\pm$0.097 & 0.5922$\pm$0.111 \\
GMC & 0.8054$\pm$0.001 & 0.7846$\pm$0.006 & \textbf{0.6548$\pm$0.008} & \textbf{0.6910$\pm$0.008} & 0.7747$\pm$0.009 & 0.7810$\pm$0.003 & 0.6978$\pm$0.004 \\
GMC+$m^{2}$-Mix & \textbf{0.8086$\pm$0.001} & \textbf{0.7882$\pm$0.005} & 0.6522$\pm$0.006 & 0.6875$\pm$0.0080 & \textbf{0.7814$\pm$0.004} & \textbf{0.7840$\pm$0.004} & \textbf{0.6988$\pm$0.003} \\ \bottomrule
\end{tabular}
\label{tab:results_gmc_f1}
\end{adjustbox}
\end{table*}

\begin{figure}[thpb!]
  \centering
    \subfigure[MulT]{\includegraphics[width=0.32\linewidth]{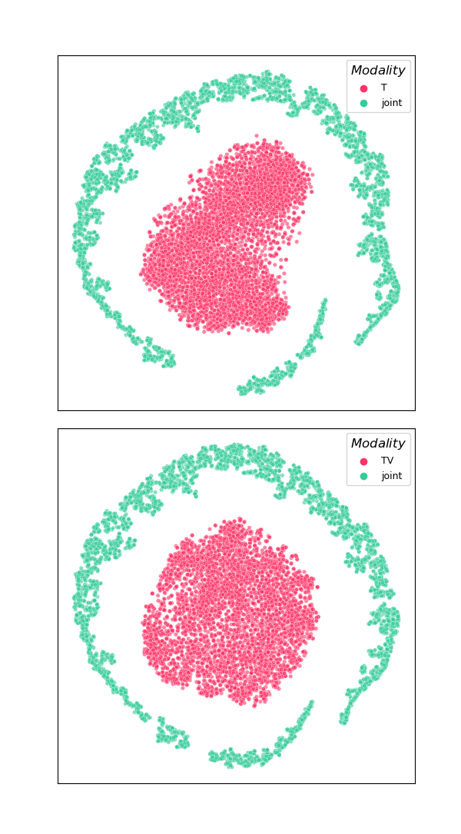}}
    \subfigure[GMC]{\includegraphics[width=0.32\linewidth]{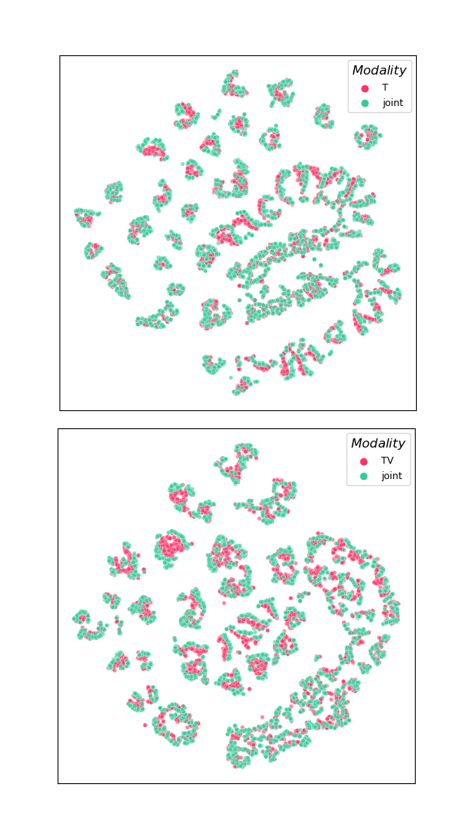}}
    \subfigure[GMC+$m^{2}$-Mix]{\includegraphics[width=0.32\linewidth]{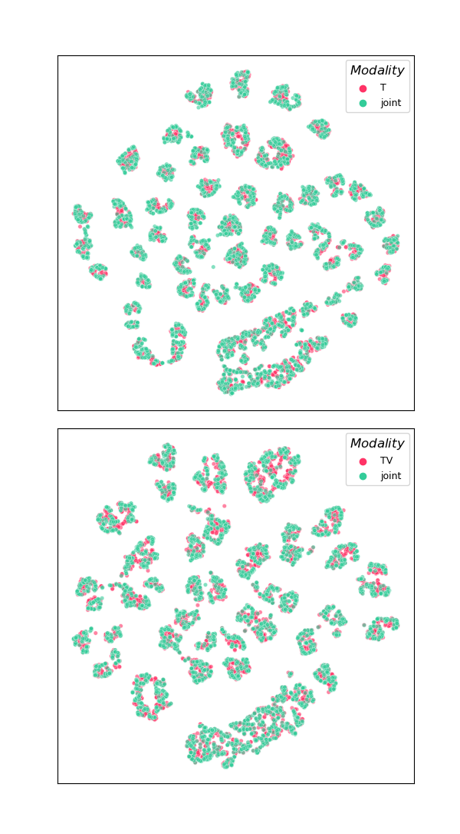}}
    
    \caption{t-SNE~\cite{van2008visualizing} for CMU-MOSEI, which has textual (T), visual (V), and audio (A) modalities. Top row represents when the only textual (T) information is given, and the bottom row corresponds to when the textual (T) and visual (V) information are given. The pink and green color denotes the embedding of partial and joint modality.}
    \label{fig:tsne_gmc_full}
\end{figure}

Fig. \ref{fig:tsne_gmc_full} shows the embedding t-SNE of each method given one (top row) or two (bottom row) modalities in test-time. Compared with MulT, GMC strongly aligns the embedding of partial and joint modality based on its explicit enforcement, and the alignment is further enhanced by the aid of $m^2$-Mix, which results in superior performance (in Tab. 7 of the main paper, Tab. \ref{tab:results_gmc_f1} of supplementary) when only partial (missing modality) information is given during test-time. These results justify the use of our $m^2$-Mix for robust multi-modal representation learning under missing modality scenarios.

\newpage
\subsection{Effect of $m^2\text{-Mix}$ on Contrastive Learning}
This section provides a more detailed analysis of $m^2$-Mix. Specifically, we present (i) the proportion of negative pairs (original and mixed) that exceed the similarity between that of positive pairs (See Fig. \ref{app:fig:coco_ihnp} - \ref{app:fig:flickr_thnp}), and (ii) the similarity comparison between positive and original negative pairs with and without $L_{m^2\text{-Mix}}$ (See Fig. \ref{app:fig_simplot_coco} and \ref{app:fig_simplot_flickr}). All results are from cross-modal retrieval with CLIP ViT-B/32 on Flickr30k and MS COCO.

Fig. \ref{app:fig:coco_ihnp} - \ref{app:fig:flickr_thnp} show the average proportion of in-batch negative pairs' similarities that exceeds the similarities of positive pairs during training iterations (dataset: MS COCO - Fig. \ref{app:fig:coco_ihnp} and \ref{app:fig:coco_thnp}, Flickr30k - Fig. \ref{app:fig:flickr_ihnp} and \ref{app:fig:flickr_thnp}, similarity computation: $I$-to-mixed - Fig. \ref{app:fig:coco_ihnp} and \ref{app:fig:flickr_ihnp}, $T$-to-mixed - Fig. \ref{app:fig:coco_thnp} and \ref{app:fig:flickr_thnp}). 

\begin{minipage}{\textwidth}
\begin{minipage}[htpb!]{0.475\textwidth}
\centering
\includegraphics[width=\textwidth]{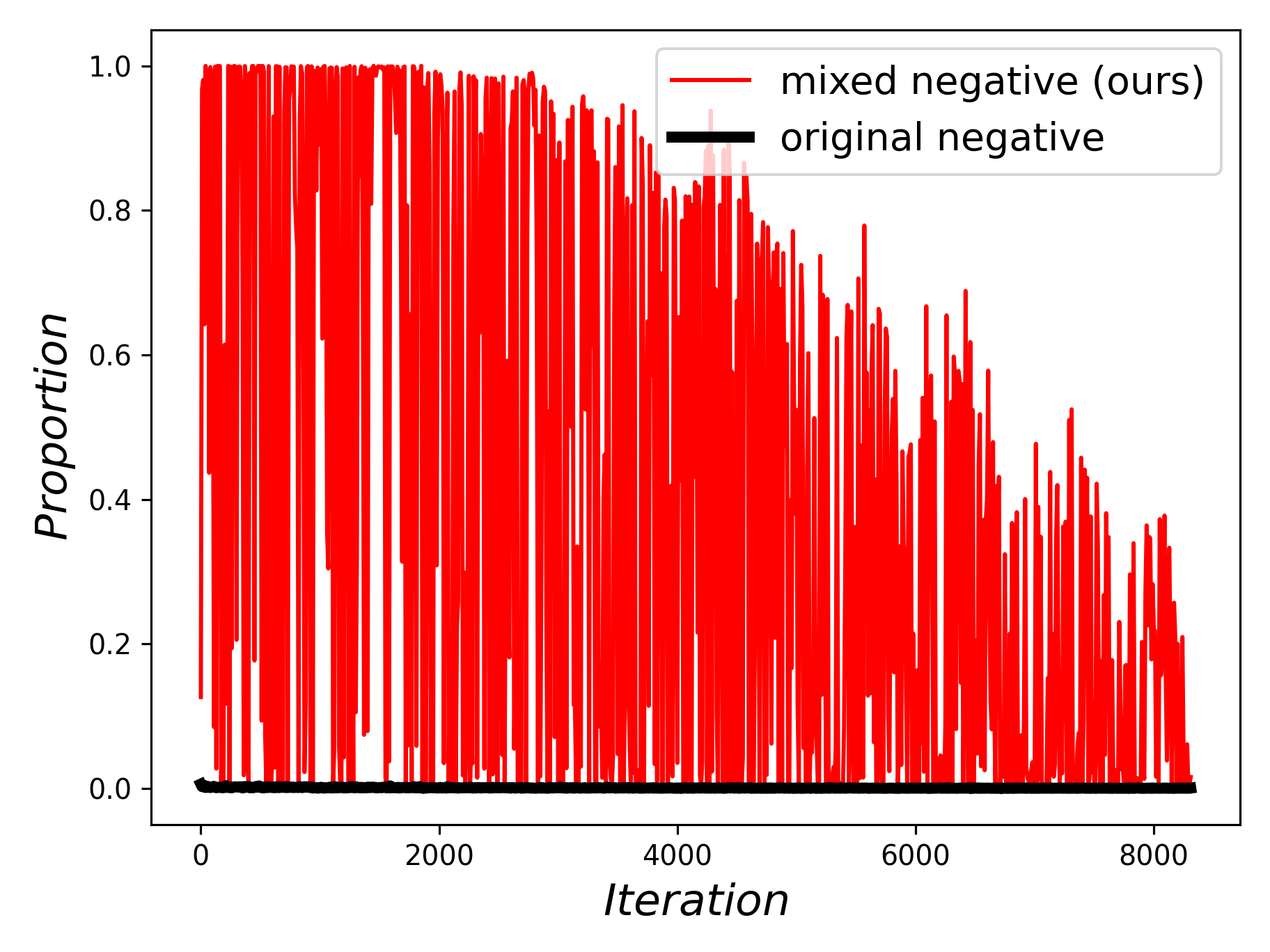}
\captionof{figure}{Hard negative proportion by $I$-to-mixed samples' similarities on MS COCO.}
\label{app:fig:coco_ihnp}
\end{minipage}
\hfill
\begin{minipage}[htpb!]{0.475\textwidth}
\centering
\includegraphics[width=\textwidth]{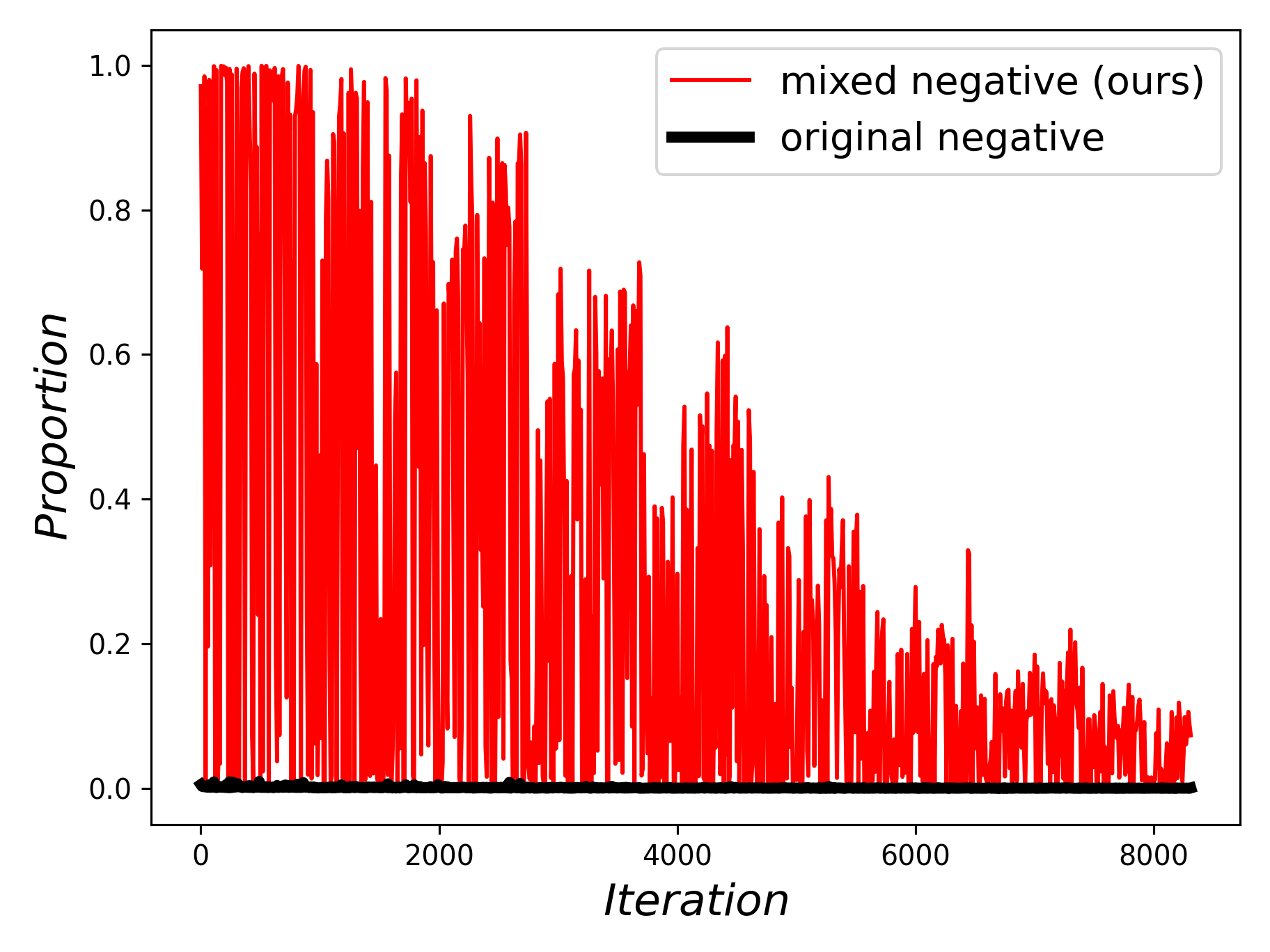}
\captionof{figure}{Hard negative proportion by $T$-to-mixed samples' similarities on MS COCO.}
\label{app:fig:coco_thnp}
\end{minipage}
\end{minipage}

\begin{minipage}{\textwidth}
\begin{minipage}[htpb!]{0.475\textwidth}
\centering
\includegraphics[width=\textwidth]{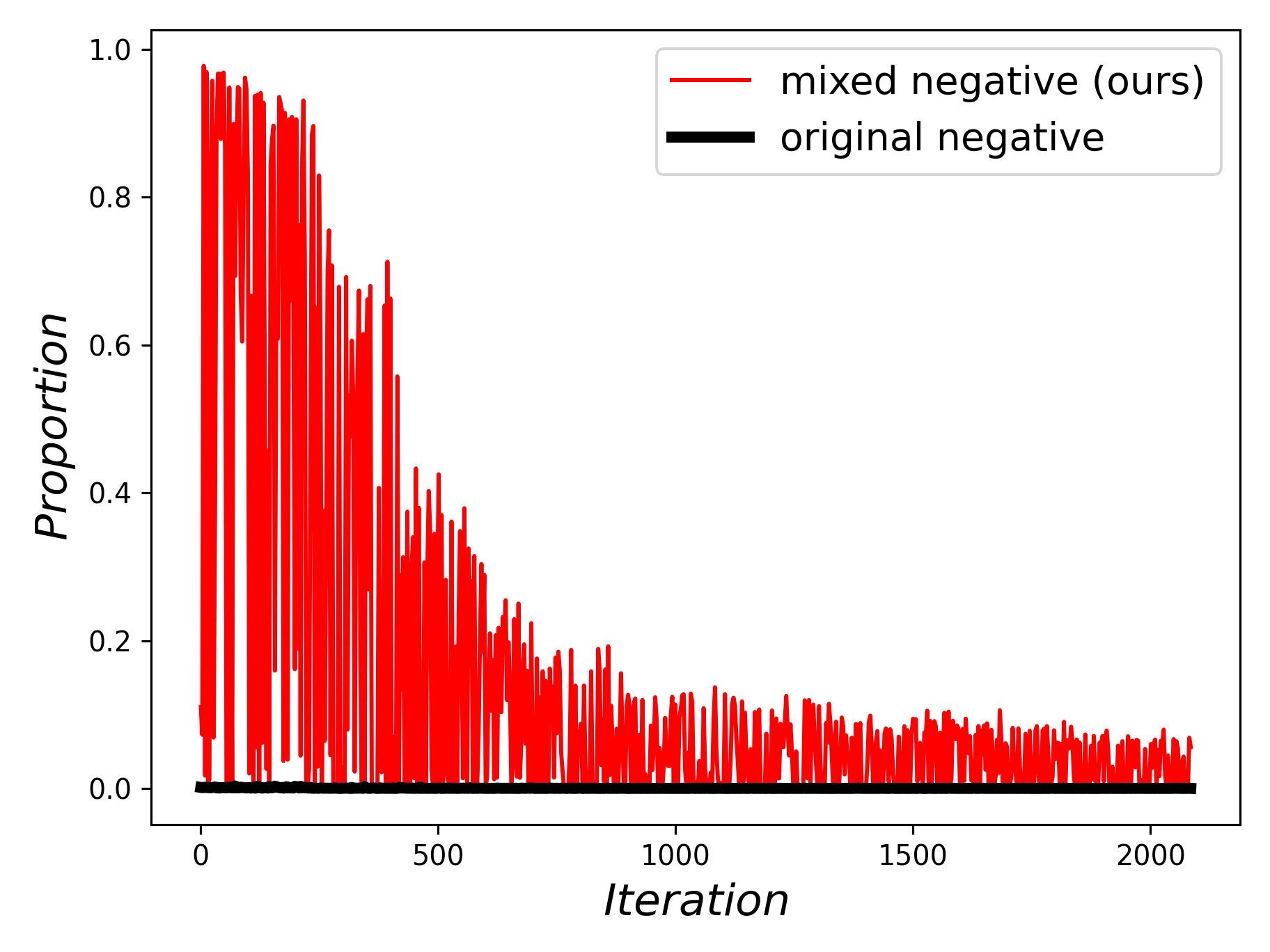}
\captionof{figure}{Hard negative proportion by $I$-to-mixed samples' similarities on Flickr30k.}
\label{app:fig:flickr_ihnp}
\end{minipage}
\hfill
\begin{minipage}[htpb!]{0.475\textwidth}
\centering
\includegraphics[width=\textwidth]{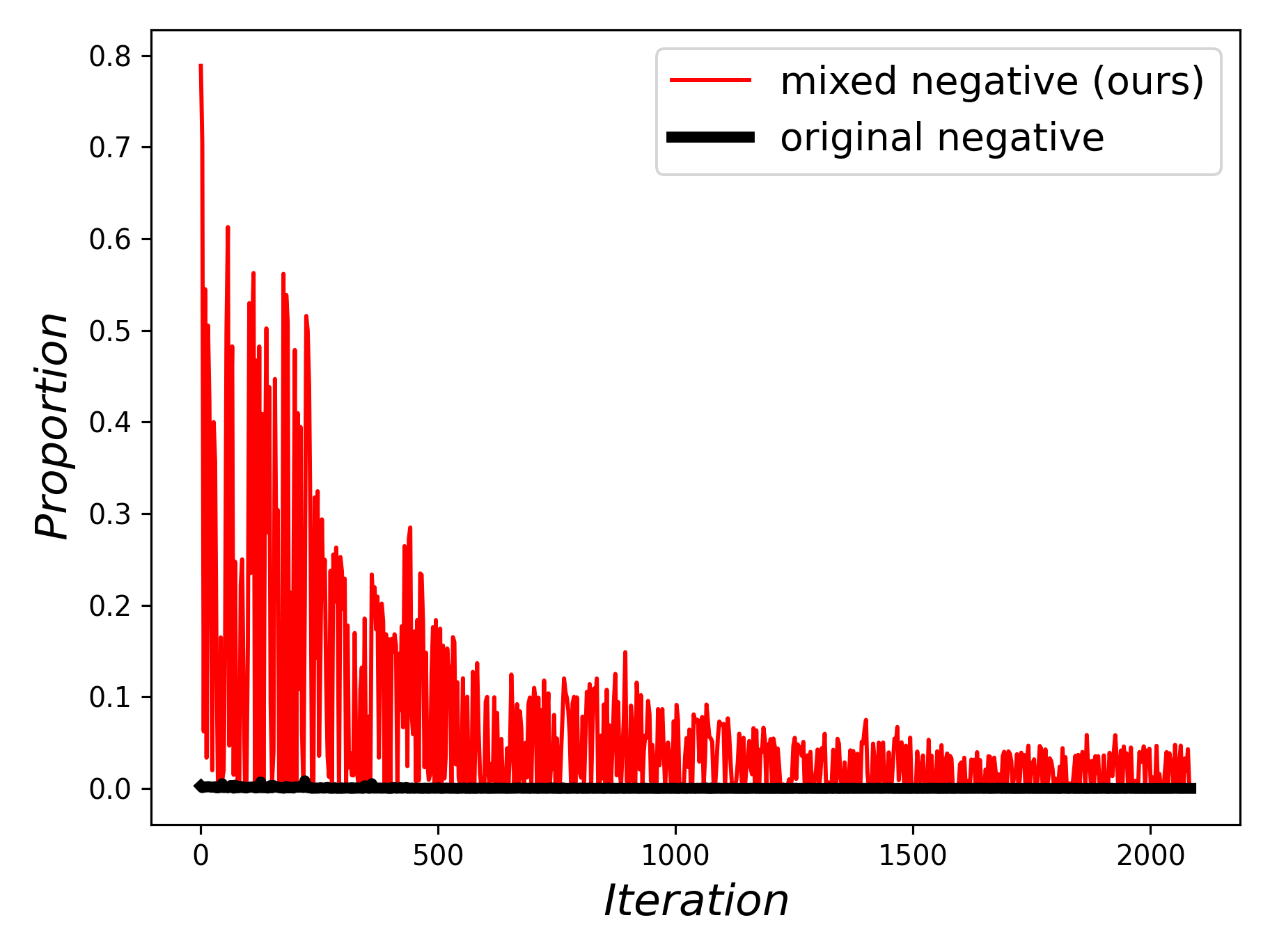}
\captionof{figure}{Hard negative proportion by $T$-to-mixed samples' similarities on Flickr30k.}
\label{app:fig:flickr_thnp}
\end{minipage}
\end{minipage}

When we train our model with learning objective $L_{\text{CLIP}}+L_{m^2\text{-Mix}}$, the similarities of negative pairs by our $m^2$-Mix significantly higher than not only original negatives' similarities but also positive similarities. Especially in the early training iterations, the proportion is about one in Fig \ref{app:fig:coco_ihnp} - \ref{app:fig:flickr_ihnp}, i.e., almost all of the in-batch negative pairs have higher similarities than positive ones. These results advocate our assumption in proof of Proposition 4.2. (of the main paper and Proposition 1 of this supplementary). That is, $m^2$-Mix-generated negative pairs empirically have higher similarities than positive ones. Even though such a hard negative proportion is decaying as the training progresses (by pursuing alignment), it still has not vanished to zero, i.e., uniformity is encouraged until the end of training while the magnitude is getting weaker.

Next, in Fig. \ref{app:fig_simplot_coco} left and \ref{app:fig_simplot_flickr} left, we present the in-batch averaged pairwise similarity from normal negative pairs and positive pairs, which contribute to the computation of $L_{\text{CLIP}}$. We evaluate the similarities under two scenarios whether our $L_{m^2\text{-Mix}}$ is adopted together with $L_{\text{CLIP}}$ (dashed line) or not (solid line). Among negative pair similarities, we only consider that of the top-1 hardest negative pair that is related to the definition of relative alignment (Eq. 2 in the main paper) and contrastive loss in an asymptotic scheme (Theorem 4.2. in the main paper).

As we can see in both Fig. \ref{app:fig_simplot_coco} and \ref{app:fig_simplot_flickr}, the top-1 \textbf{original negatives'} similarities and positives' similarities are increased by using $m^2$-Mix, even though the hard negatives generated by $m^2$-Mix are not explicitly consumed in $L_{\text{CLIP}}$. Consequently, the alignment is strongly enhanced (left side of Fig. \ref{app:fig_simplot_coco} and \ref{app:fig_simplot_flickr}). The results can be summarized as follows: (1) Whether $L_{m^2\text{-Mix}}$ is used or not, the similarities between positive pairs are larger than that of original negative pairs during the whole training time. (2) However, if $L_{m^2\text{-Mix}}$ is used with $L_{\text{CLIP}}$, the top-1 negative similarities are increased. (3) As a result, the similarities between positive pairs are increased to decrease the contrastive loss $L_{\text{CLIP}}$ (which converges asymptotically to triplet loss).

\begin{figure}[htpb!]
\centering
\includegraphics[width=\textwidth]{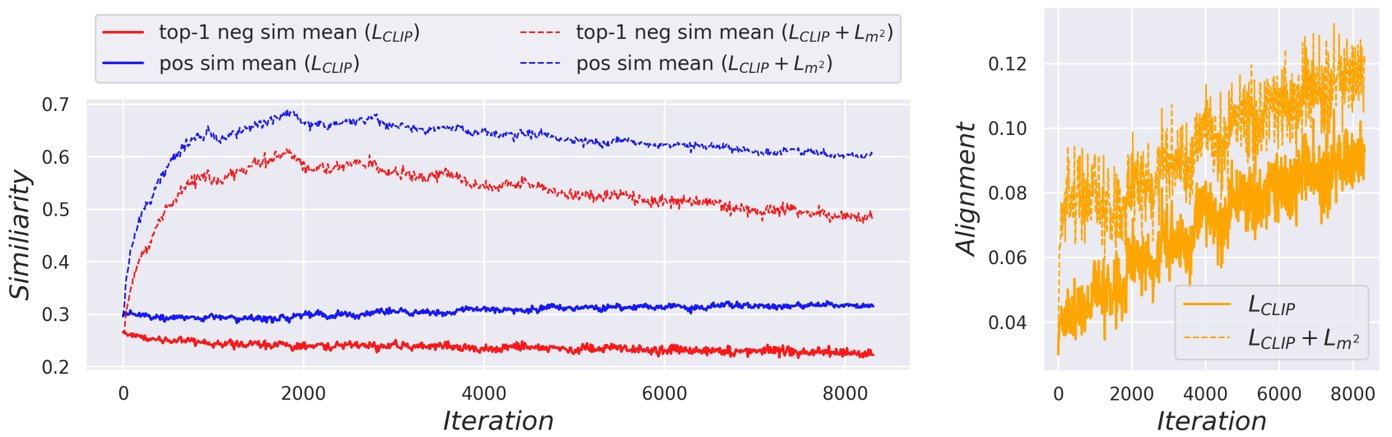}
\captionof{figure}{In-batch averaged pairwise similarity and alignment comparison over training iterations with and without $m^2$-Mix on MS COCO. (Left) similarities of original negative pairs (top-1 highest) and that of positive pairs. (Right) alignment comparison.}
\label{app:fig_simplot_coco}
\end{figure}

\begin{figure}[htpb!]
\centering
\includegraphics[width=\textwidth]{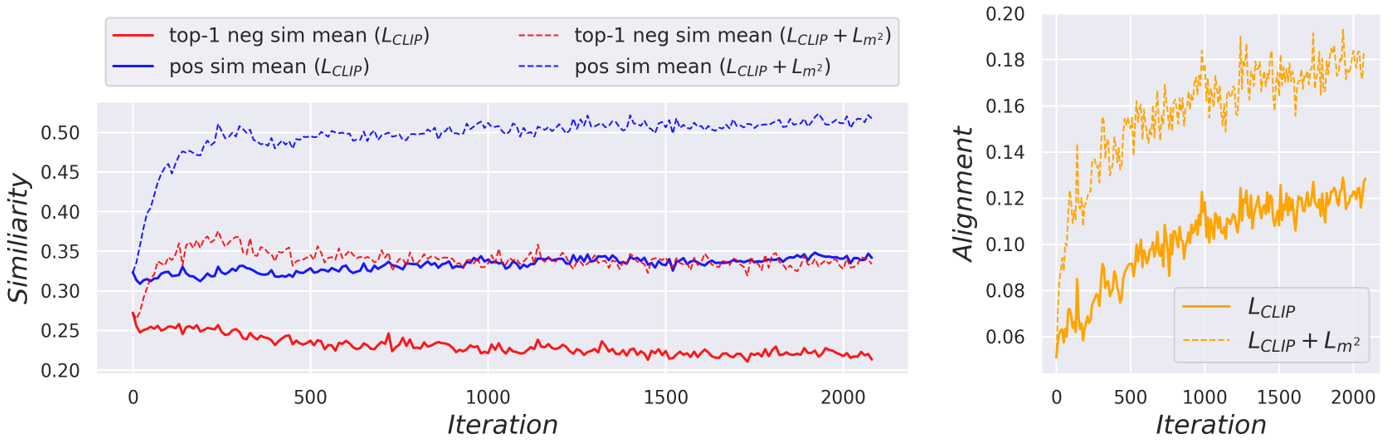}
\captionof{figure}{In-batch averaged pairwise similarity and alignment comparison over training iterations with and without $m^2$-Mix on Flickr30k. (Left) similarities of original negative pairs (top-1 highest) and that of positive pairs. (Right) alignment comparison.}
\label{app:fig_simplot_flickr}
\end{figure}

These empirical results imply that our $L_{m^2\text{-Mix}}$ implicitly helps the alignment in $L_{\text{CLIP}}$, and by minimizing $L_{\text{CLIP}} + L_{m^2\text{-Mix}}$, we can deal with both alignment and uniformity better (Proposition 4.2. of main paper).

\newpage

\section{Proof} 
\label{appendix:proof}

\begin{theorem}[Hardness of $m^2$-Mixed samples]
Let's assume that two random variables $x_{1}$ and $x_{2}$ follow the $M_{d}(\mu_{1},\kappa)$ and $M_{d}(\mu_{2},\kappa)$, von Mises–Fisher distribution with mean direction $\mu_{1}, \mu_{2}$ and concentration parameter $\kappa$ in $\mathbb{R}^{d}$, respectively.
Let $\widetilde{x}=x_{1}+x_{2}$ and $d=2$.
Then, $D_{KL}(p(x_{1})||p(\widetilde{x})) \leq D_{KL}(p(x_{1})||p(x_{2}))$ for sufficiently large $\kappa$.
\label{eq:appendix_thm}
\end{theorem}

\begin{proof}[Proof of Theorem \ref{eq:appendix_thm}]
Let $M_{d}(\mu,\kappa) = C_{d}(\kappa)\exp(\kappa \mu^{T}x)$, where $C_{d}(\kappa)=\frac{\kappa^{d/2-1}}{(2\pi)^{(d/2)}I_{d/2-1}(\kappa)}$.
Let $x_{1} \sim M_{2}(\mu_{1},\kappa)$ and $x_{2} \sim M_{2}(\mu_{1},\kappa)$, and $I$ denotes the modified Bessel function.
From~\cite{markovic2012bearing}, $\widetilde{x} \sim M_{2}(\widetilde{\mu},\widetilde{\kappa})$ where $\widetilde{\mu}=\mu_{1}+\mu_{2}$ and $\widetilde{\kappa} = A^{-1}(A(\kappa)A(\kappa))$, approximately, where $A(\kappa)=\frac{I_{1}(\kappa)}{I_{0}(\kappa)}$ and $A^{-1}$ is its inverse. 
From~\cite{jammalamadaka2021functional}, $D_{KL}(p(x_{1})||p(x_{2})) = \kappa A(\kappa) (1 - \text{cos}(\mu_{1}-\mu_{2}))$. 
Similarly, $D_{KL}(p(x_{1})||p(\widetilde{x}))
= \log I_{0}(\widetilde{\kappa}) - \log I_{0}(\kappa) + \widetilde{\kappa} A(\widetilde{\kappa}) (1 - \text{cos}(\mu_{1}-\mu_{2}))$.
From~\cite{martin2017precise} and~\cite{olivares2018simple}, $I_{0}(\kappa)\approx \frac{\exp(\kappa)}{\sqrt{2\pi \kappa}}(1+\frac{1}{8\kappa}) \text{ and } I_{1}(\kappa)\approx \frac{\exp(\kappa)}{\sqrt{2\pi \kappa}}(1-\frac{3}{8\kappa})$ for sufficiently large $\kappa$.
Therefore, $D_{KL}(p(x_{1})||p(\widetilde{x})) \leq D_{KL}(p(x_{1})||p(x_{2}))$ for sufficiently large $\kappa$.
\end{proof}

\begin{proposition}[Limiting behavior of $\bL_{\text{CLIP}}$ with $m^2$-Mix]
For sufficiently large $M$, as the temperature of contrastive loss $\tau \to 0^{+}$, the $\bL_{\text{CLIP}}$ and $\bL_{m^{2}\text{-Mix}}$ converge to the triplet loss with zero-margin (i.e., corresponding to negative \text{Alignment}) and negative \text{Uniformity}, respectively. That is: 
$\lim_{\tau \to 0^{+}} \bL_{\text{CLIP}} + \bL_{m^{2}\text{-Mix}} \simeq -(\text{Alignment}+\text{Uniformity})$ 
\label{eq:appendix_prop} 
\end{proposition}

\begin{proof}[Proof of Proposition \ref{eq:appendix_prop}]
As noted in Section 3., given a training batch $\{x_i,y_i\}_{i=1}^{M}$ of image-text pairs and image and text encoders $f(\cdot;\theta_1)$ and $g(\cdot;\theta_2)$, and the $L_{2}$-normalized image-text embeddings are $(I,T)=\{f(x_i;\theta_1), g(y_i;\theta_2)\}_{i=1}^{M}$. Then, for $\theta=\{\theta_1, \theta_2\}$ and a scalar $\tau > 0$, a standard contrastive loss adopted by CLIP is formulated as:

\begin{align}
    C(I,T;\theta) = \frac{1}{M}\sum_{i=1}^{M}-\log\frac{\exp((I_i\cdot T_i)/\tau)}{\sum_{j=1}^{M}\exp{((I_i\cdot T_j)/\tau})} \label{a_eq:onesied_lclip}
\end{align}

Although the actual loss function is constructed by averaging two-way contrastive losses, i.e., $L_{CL}=\frac{1}{2}(C(I,T;\theta)+C(T,I;\theta))$, here, we consider only one-way contrastive loss $C(I,T;\theta)$ for simplicity. It is easily shown the derivation of the other side by changing the order of $I$ and $T$.

During pre-training, the temperature $\tau$ of CLIP converges to 0.01, which is a significantly small value that makes the pair-wise similarities sharp. Thus, it will be reasonable to consider an extreme case: when $\tau \to 0^{+}$. In this case, we can approximate the $C(I,T;\theta)$ as follow:

\begin{align}
    C(I,T;\theta) &= \lim_{\tau \to 0^{+}} \frac{1}{M}\sum_{i=1}^{M}-\log\frac{\exp(I_i\cdot T_i/\tau)}{\sum_{j=1}^{M}\exp{(I_i\cdot T_j/\tau})} \label{eq:lclip_approx} \\ \nonumber
          &= \lim_{\tau \to 0^{+}} \frac{1}{M}\sum_{i=1}^{M} - (I_i\cdot T_i)/\tau +\log\left[ \exp{(I_i\cdot T_i/\tau)} + \sum_{j \neq i}\exp{(I_i\cdot T_j/\tau)} \right] \\ \nonumber
          &= \lim_{\tau \to 0^{+}} \frac{1}{M}\sum_{i=1}^{M} \log\left[ 1 + \sum_{j \neq i}\exp{((I_i\cdot T_j) - (I_i\cdot T_i) /\tau)} \right] \\ \nonumber
          &= \lim_{\tau \to 0^{+}} \frac{1}{M}\sum_{i=1}^{M} \log\left[ 1 + \sum_{j \in \mathcal{J}(i,I,T)}\exp{((I_i\cdot T_j) - (I_i\cdot T_i) /\tau)} \right] \tag{where $\mathcal{J}(i,I,T) := \{j | (I_i\cdot T_j) > (I_i\cdot T_i)\}$} \\ \nonumber
          &= \lim_{\tau \to 0^{+}} \frac{1}{M}\sum_{i=1}^{M} \frac{1}{\tau} \max \left[ \max_{j}(I_i\cdot T_j) - (I_i\cdot T_i), 0\right] \\ \nonumber
          &\simeq \lim_{\tau \to 0^{+}} -\text{Alignment}(I,T;\theta) \tag{when $\max_{j}(I_i\cdot T_j) - (I_i\cdot T_i) > 0$ and $M$ is sufficiently large}
\end{align}
\noindent
where $\max_{j}(I_i\cdot T_j)$ denotes the maximum similarity among negative pairs. From this derivation, we show that the multi-modal contrastive loss only considers the top-1 hardest negative pairs to positive ones like a triplet loss. As a result, minimizing this loss function is equivalent to maximizing the relative alignment that we newly defined in this paper (in Section 3. Eq. 2 of main paper), e.g., for sufficiently large $M$, $\displaystyle \minimize_{\theta} \; C(I,T;\theta) \; \equiv \; \maximize_{\theta} \; \text{Alignment}(I,T;\theta)$. Note that, however, if there are no hard negatives that have higher similarity than positives, the above loss term does not give a meaningful learning signal because the loss already approaches zero. This issue is resolved by $m^2$-Mix, which generates hard negatives explicitly.

Next, when we consider the $m^2$-Mix-based contrastive loss under same case ($\tau \to 0^{+}$), we can obtain  another approximation like below:

\begin{align}
C_{m^{2}\text{-Mix}}(I,T;\theta) &= \lim_{\tau \to 0^{+}} \frac{1}{M}\sum_{i=1}^{M}-\log\frac{\exp(I_i\cdot T_i/\tau)}{\exp{(I_i\cdot T_i/\tau)} + \sum_{j \neq i}\exp{(I_i\cdot m_{\lambda}(I_i,T_j)/\tau})} \label{eq:lm2mix_approx} \\ \nonumber
&= \lim_{\tau \to 0^{+}} \frac{1}{M}\sum_{i=1}^{M} - (I_i\cdot T_i)/\tau +\log\left[ \exp{(I_i\cdot T_i/\tau)} + \sum_{j \neq i}\exp{(I_i\cdot  m_{\lambda}(I_i,T_j)/\tau)} \right] \\ \nonumber
&= \lim_{\tau \to 0^{+}} \frac{1}{M}\sum_{i=1}^{M} \log\left[ 1 + \sum_{j \neq i}\exp{((I_i\cdot  m_{\lambda}(I_i,T_j) - (I_i\cdot T_i)) /\tau)} \right] \\ \nonumber
&= \lim_{\tau \to 0^{+}} \frac{1}{M}\sum_{i=1}^{M} \log\left[ 1 + \sum_{j \neq i}\exp{(I_i\cdot m_{\lambda}(I_i,T_j)/\tau)} \right] \tag{by assuming $I_i\cdot  m_{\lambda}(I_i,T_j) > I_i\cdot T_i \; \text{for all} \; j \neq i$} \\ \nonumber
&= \lim_{\tau \to 0^{+}} \frac{1}{M}\sum_{i=1}^{M} \log \sum_{j \neq i}\exp{(I_i\cdot m_{\lambda}(I_i,T_j)/\tau)} \\ \nonumber
&\simeq \lim_{\tau \to 0^{+}} -\text{Uniformity}(I,m_{\lambda}(I_i,T_j);\theta) \tag{for sufficiently large $M$}
\end{align} 
\noindent

Where $m_{\lambda}(\cdot,\cdot)$ is the geodesic Mixup operation with mixing ratio $\lambda$. Based on above Eq. \ref{eq:lm2mix_approx}, we argue that our $m^2$-Mix asymptotically maximizes the uniformity given sufficiently in-batch large samples $M$, i.e., $\displaystyle \minimize_{\theta} \; C_{m^{2}\text{-Mix}}(I,T;\theta) \; \equiv \; \maximize_{\theta} \; \text{Uniformity}(I,m_{\lambda}(I_i,T_j);\theta)$.

To complete the proof, consider the two-side contrastive losses $L_{\text{CLIP}} = \frac{1}{2}(C(I,T;\theta) + C(T,I;\theta))$ and $L_{m^2\text{-Mix}} = \frac{1}{2}(C_{m^{2}\text{-Mix}}(I,T;\theta) + C_{m^{2}\text{-Mix}}(T,I;\theta))$. Then, by minimizing the combination of the standard contrastive loss $L_{\text{CLIP}}(\theta)$ (Eq. \ref{eq:lclip_approx}) and $L_{m^{2}\text{-Mix}}(\theta)$ (Eq. \ref{eq:lm2mix_approx}), we can approximately maximize both Alignment and Uniformity.
\end{proof}

\section{Limitation} 
While our method showcases its versatility on diverse downstream tasks, it induces computational overhead to some extent. In this work, we observed that uniformity and alignment (our modified formulation) are both somewhat correlated with downstream performance, so we argued that uniformity and alignment are crucial in multi-modal representation learning as well as uni-modal representation learning~\cite{wang2020understanding, wang2021understanding}. However, those two metrics are not absolute ones for various situations, e.g., when the modality-specific unique information is important, the higher uniformity and alignment can cause the loss of modality-specific information. The increased computational cost for additional contrastive loss terms is another limitation.

\end{document}